%% file: main.tex
\documentclass{article} 
\usepackage{iclr2027_conference,times}
\iclrfinalcopy 
\input{math_commands.tex}

\usepackage{hyperref}
\hypersetup{hidelinks}
\usepackage{url}

\usepackage{xcolor}
\definecolor{linkblue}{rgb}{0.10,0.25,0.60}
\usepackage{booktabs}
\usepackage{amsmath}
\usepackage{amssymb}
\usepackage{amsthm}
\usepackage{graphicx}
\usepackage{algorithm}
\usepackage{algorithmic}

\newtheorem{theorem}{Theorem}
\newtheorem{proposition}[theorem]{Proposition}

\newtheorem{corollary}[theorem]{Corollary}
\theoremstyle{definition}

\newcommand{\kl}{\operatorname{kl}}
\newcommand{\Mzero}{\textsf{M0}}
\newcommand{\Mone}{\textsf{M1}}
\newcommand{\Mtwo}{\textsf{M2}}
\newcommand{\VigilLR}{\textup{\textsc{Vigil-LR}}}
\newcommand{\VigilCap}{\textup{\textsc{Vigil-Cap}}}
\newcommand{\VigilCell}{\textup{\textsc{Vigil-Cell}}}
\newcommand{\TransportFirst}{\textup{\textsc{Transport-First}}}
\newcommand{\up}{$\uparrow$}
\newcommand{\down}{$\downarrow$}

\title{When Can Old Evaluations Certify a New Model? \\ Label-Efficient Release Decisions under Evaluator Drift}

\author{\textbf{Joyanta Jyoti Mondal}\textsuperscript{1,*},
\textbf{Mridul Banik}\textsuperscript{2,*},
\textbf{Md. Shifatul Ahsan Apurba}\textsuperscript{3,*},
\\
\textbf{Md Masud Al Mahmud}\textsuperscript{4}
\\
\textsuperscript{1}Department of Computer and Information Sciences, University of Delaware, USA\\
\textsuperscript{2}Department of Biomedical Informatics and Data Science,
University of Alabama at Birmingham, USA\\
\textsuperscript{3}Luddy School of
Informatics, Computing, and Engineering, Indiana University Indianapolis, USA\\
\textsuperscript{4}Department of Computer Science and Engineering, BRAC University, Bangladesh\\
\small{
\textsuperscript{\textbf{*}}Equal Contribution.
\textbf{Correspondence:}
\href{mailto:joyanta@udel.edu}{joyanta@udel.edu}
}
}

\begin{document}
\maketitle

\begin{abstract}
Releasing a model update requires certifying that its current-population risk stays below a threshold. Trusted labels are expensive, while a cheap evaluator, such as an LLM judge, scores every example. Reusing evaluator errors from earlier audits is tempting, but when may such evidence replace current labels? It depends on the status of history. If the errors can change invisibly, no label-free test detects the change, and every valid, useful certifier must keep buying labels at a rate we characterize; if a bound on the change is assumed, label-free certification is valid at an explicit error cost. For the middle ground, where history is informative but untrusted, we propose \emph{portfolio vigilance}, a sequential certifier mixing a betting expert guided by history with one that learns only from current labels; history affects only how it bets, so validity holds for any history. The contribution is not prior-informed betting or expert mixtures, but separating history that may enter validity from history that may only guide label collection. In a canonical model, accurate history shortens decisions but never raises the evidence growth rate; stale history can destroy it. On held-out CIFAR-10N and DICES-990 data, portfolio vigilance needs 0.465 (95\% CI $[0.327,0.575]$) and 0.740 ($[0.618,0.877]$) times the labels of a matched prediction-powered monitor, with no observed false certification, and fewer labels on all six external blocks. Under corrupted advice it stays within 8.0\% of its better component, while trusting history alone costs up to 1.66 times as much. In post-confirmatory repeated-judge experiments on DICES-990 and ToxicChat, changing a fixed LLM judge's rubric moves its scores beyond run-to-run variation; the portfolio then needs 0.790 ($[0.667,0.909]$) and 0.631 ($[0.520,0.770]$) times the labels of the matched monitor, and fewer than trusting history alone.
\end{abstract}

\section{Introduction}\label{sec:introduction}

Suppose a team fine-tunes a deployed model and must decide whether the new version may be released. The question is statistical: does the candidate's risk on the population it serves stay below a threshold, at a stated error level? Trusted labels, such as expert safety ratings, answer it but are slow and expensive. Cheap evidence is abundant: a learned evaluator or a large language model (LLM) judge can score every example, and earlier audits leave a ledger of examples that carry both cheap scores and trusted labels. Prediction-powered evaluation shows how to combine the two \citep{angelopoulos2023ppi}. Writing $Y$ for the trusted loss and $Q$ for its cheap prediction, the risk splits into a free part and a paid part,
\begin{equation}\label{eq:decomposition}
    \underbrace{\E_T[Y]}_{\text{target risk}}
    = \underbrace{\E_T[Q]}_{\text{recomputable without labels}}
    + \underbrace{\E_T[Y-Q]}_{\text{needs trusted labels}} .
\end{equation}
Only the mean residual $E=Y-Q$, the evaluator's average error, needs trusted labels. A historical ledger measures exactly this quantity, so reusing it is tempting: the model and the population change, but perhaps the evaluator's errors do not.

Whether this reuse is legitimate is the question of this paper. Existing methods estimate the residual from current labels \citep{angelopoulos2023ppi,zhang2026pprm}, calibrate the current evaluator against human labels \citep{feng2026noisy}, or let side information steer anytime-valid evidence without affecting its validity \citep{waudbysmith2024betting,chen2026sim2real}. None of them says when historical residual evidence may replace current labels and when it may only guide how they are collected. Nor does any certifier capture the benefit of informative but untrusted history while avoiding the cost of stale history, which can be large because changing the evaluator changes the measurement itself \citep{yang2026judgechanges}.

We answer both questions by distinguishing three statuses of historical evidence (Figure~\ref{fig:regimes}). Under \emph{hidden residual drift} (\Mzero), the evaluator's error can change without any trace in the cheap data, so history need not constrain the present. Under \emph{trusted transport} (\Mone), a declared bound limits how far the error can move, and history may enter the validity argument directly. Under \emph{untrusted advice} (\Mtwo), history may predict the present error but is not assumed correct, so it may only steer how evidence is collected. History can replace current labels only under \Mone, and because \Mzero{} and \Mone{} look identical in cheap data, that status must be assumed, not learned. For \Mtwo{}, we propose \emph{portfolio vigilance}. It samples the target pool at random, runs one betting process guided by history and one that learns the error from purchased labels alone, and certifies from their fixed mixture. Accurate history then saves labels, while stale history costs little and never affects validity.

\begin{figure}[t]
\centering
\includegraphics[width=\textwidth]{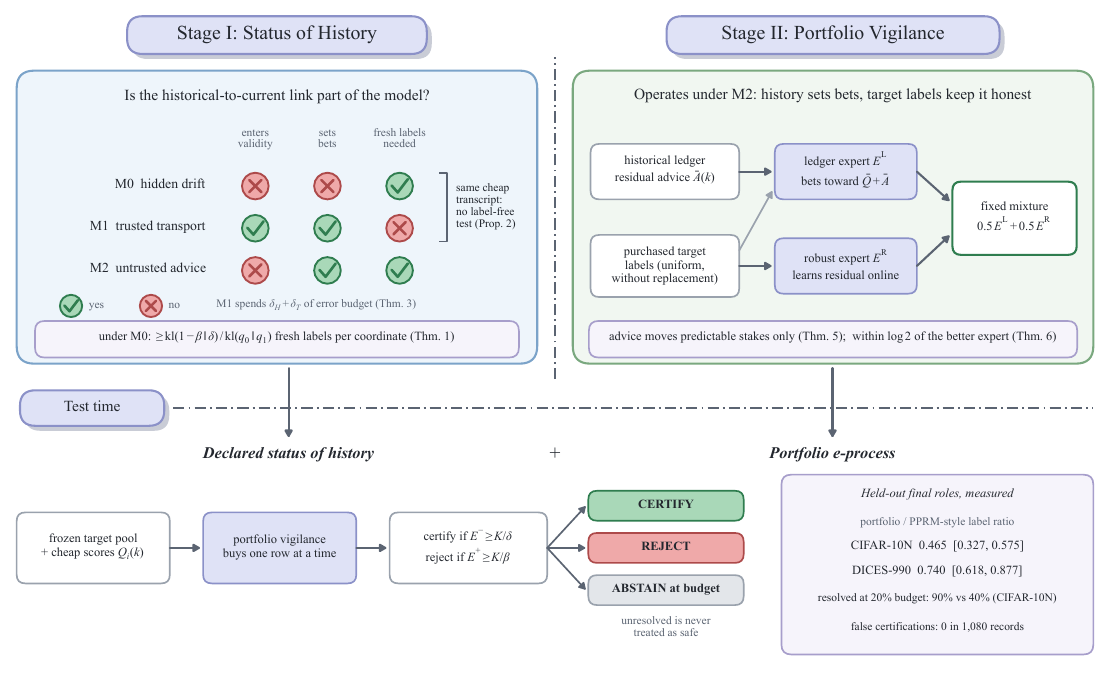}
\caption{Overview. \textbf{Stage I:} hidden drift (\Mzero) and trusted transport (\Mone) are indistinguishable from cheap data (Proposition~\ref{prop:nolabelfree}); \Mzero{} forces fresh labels (Theorem~\ref{thm:vigilance}), while \Mone{} certifies without them at a declared error cost (Theorem~\ref{thm:trustedtransport}). \textbf{Stage II:} under untrusted advice (\Mtwo), a ledger expert and a robust expert bet on each purchased label, and their equal-weight mixture is valid for any advice (Theorem~\ref{thm:validity}) and within $\log2$ evidence of the better expert (Theorem~\ref{thm:bestexpert}). \textbf{Test time:} rows are bought until each candidate is certified, rejected, or unresolved.}
\label{fig:regimes}
\end{figure}

Our contributions are:
\begin{itemize}
    \item \textbf{Transport or vigilance.} We prove that hidden drift and trusted transport are indistinguishable from cheap data, give a tight lower bound on the fresh labels honest certification needs, and price the error cost of a trusted bridge.
    \item \textbf{A phase diagram for untrusted history.} We characterize when historical advice helps, stops helping, or destroys evidence, and quantify the price of hedging.
    \item \textbf{Portfolio vigilance.} We propose a finite-population certifier that is valid under arbitrary advice and optimal among fixed mixtures in worst-case evidence regret.
    \item \textbf{Evaluation.} In a prespecified evaluation on held-out DICES-990 and CIFAR-10N roles, we compare against matched baselines, with external, corrupted-advice, and sensitivity analyses; a post-confirmatory study with repeated LLM judges on DICES-990 and ToxicChat tests evaluator drift directly.
\end{itemize}
Prediction-powered inference, e-values, betting confidence sequences, prior-informed betting, and expert mixtures all predate this work; the contribution is the information-model distinction, its consequences, and a certifier that makes it operational.

\section{Related work}\label{sec:related}

\paragraph{Prediction-powered and anytime-valid inference.} Prediction-powered inference corrects a cheap predictor's mean with labeled residuals \citep{angelopoulos2023ppi,angelopoulos2023ppiplus,zrnic2024cross,zrnic2024active}. Its anytime-valid and monitoring variants include prediction-powered e-values \citep{csillag2025ppevalues}, prediction-powered risk monitoring (PPRM) \citep{zhang2026pprm}, and judge auditing for best-arm identification \citep{ao2026bestarm}. Betting confidence sequences \citep{howard2021confseq,waudbysmith2024betting,waudbysmith2021wor,ville1939} underlie risk monitors \citep{podkopaev2022tracking,timans2025monitoring,koebler2025monitoring,xu2024active}. Simulator-guided betting also shows that side information can steer the bets of an anytime-valid certificate whose validity holds for any simulator bank \citep{chen2026sim2real}. All of these estimate the residual from current labels; none says when a historical residual may replace them.

\paragraph{Changing evaluators.} Noisy-but-Valid calibrates a current judge against human labels \citep{feng2026noisy}. Judge-replacement audits show that changing the judge changes the measurement \citep{yang2026judgechanges,li2026whodrifted}, and disagreement-based auditing exploits observable disagreement between models \citep{balachandran2026disagreement}. These works study the current evaluator or a visible signal; we ask whether residual evidence from a past evaluator still describes the present one.

\paragraph{Transport and adaptation.} Transport results assume a known shift or a declared slack \citep{tibshirani2019covshift,podkopaev2021labelshift,doula2026transfer,klivans2024testable}. Honest procedures, however, cannot adapt to unknown structure for free \citep{low1997,cailow2004,armstrong2018,mazzetto2023drift}. Our lower bound uses the change-of-measure argument of \citet{kaufmann2016}. We study the case these results leave open: a historical residual whose drift may leave no trace in any cheap observable. Appendix~\ref{app:related} gives a claim-level comparison.

\section{Current-risk certification and information models}\label{sec:setup}

Table~\ref{tab:notation} in the appendix collects the notation for reference. We audit a finite target pool of $N$ rows and $K$ prespecified candidate claims. For row $i$ and candidate $k$, let $Y_i(k)\in[0,1]$ be a trusted loss and $Q_i(k)\in[0,1]$ a cheap prediction of it, available for the whole pool before any trusted outcome is purchased. The current finite-population risk is
\begin{equation}\label{eq:risk}
    R(k)=\frac1N\sum_{i=1}^{N}Y_i(k),\qquad Y_i(k)=Q_i(k)+E_i(k),
\end{equation}
where $E_i(k)$ is the current evaluator residual. Candidate $k$ is safe when $R(k)\le\tau_k$, with $\tau_k$ fixed in advance. A certifier purchases trusted outcomes one row at a time, and each purchased row updates every candidate defined on it. It must be \emph{valid}, certifying some unsafe candidate with probability at most $\delta$, and should be \emph{useful}, resolving safe candidates with probability at least $1-\beta$. A ledger of paired cheap and trusted outcomes from an earlier role supplies a residual estimate $\bar A(k)$ of the historical mean residual $\mu_E^{\mathrm{hist}}(k)$.

\paragraph{\Mzero: unrestricted hidden residual drift.} The evaluator's error may change while every cheap observable keeps its distribution. Historical residuals then need not constrain the current residual mean $\mu_E^{\mathrm{tar}}(k)$.

\paragraph{\Mone: trusted transport.} Before target labels are inspected, a historical procedure supplies $U_H(k)$ with $\Pr\{\mu_E^{\mathrm{hist}}(k)\le U_H(k)\}\ge1-\delta_H$, and a declared transport model supplies a radius $\rho_k$ with $\Pr\{\mu_E^{\mathrm{tar}}(k)\le\mu_E^{\mathrm{hist}}(k)+\rho_k\}\ge1-\delta_T$. The bridge is part of the model class, so it may enter validity directly. Zero-fresh-label certification is legitimate whenever
\begin{equation}\label{eq:bridge}
    \overline Q_{\mathrm{tar}}(k)+U_H(k)+\rho_k\le\tau_k .
\end{equation}
As $\rho_k$ grows, this zero-label region shrinks.

\paragraph{\Mtwo: untrusted advice.} Historical residuals may predict current residuals but are not assumed correct; they may set predictable bets but may not alter the null, the threshold, or the error guarantee.

\section{Why vigilance is necessary}\label{sec:vigilance}

The difficulty under \Mzero{} is that a safe world and an unsafe world can produce identical cheap data, so only trusted labels tell them apart, and a certifier that is valid in both must buy some. We make this precise in a \emph{Bernoulli trusted-loss submodel} of \Mzero{} with $M$ monitored coordinates, for example deployment blocks times independently auditable strata. The cheap prediction is identically zero, $Q\equiv0$, so the residual equals the trusted loss, and each trusted-label request at coordinate $i$ reveals an independent draw $Y\sim\mathrm{Ber}(q)$. The submodel lies in \Mzero{} because \Mzero{} does not restrict how the residual law changes. The composite safe class has mean at most $q_0$ and the unsafe class at least $q_1$, $0<q_0<q_1<1$. Under the stationary boundary-safe law $P_0$ every queried loss is $\mathrm{Ber}(q_0)$, and $P_i$ makes coordinate $i$ alone $\mathrm{Ber}(q_1)$. The cheap transcript has the same law under $P_0$ and every $P_i$. Let $C_i=1$ when coordinate $i$ is certified and let $N_i$ be the number of labels queried there. We require validity across all coordinates (honesty) and certification of every coordinate of a safe stream (liveness),
\begin{equation}\label{eq:honesty}
    \Pr\big(\exists i:\ i\text{ unsafe and }C_i=1\big)\le\delta,
    \qquad
    P_0\big(C_i=1\text{ for every }i\big)\ge1-\beta,
\end{equation}
with $1-\beta>\delta$ and $\kl(a\|b)=a\log\frac ab+(1-a)\log\frac{1-a}{1-b}$.

\begin{theorem}[Price of vigilance]\label{thm:vigilance}
Every auditor satisfying Eq.~\ref{eq:honesty} obeys, for every coordinate $i$,
\begin{equation}
    \E_0[N_i]\ \ge\ \frac{\kl(1-\beta\|\delta)}{\kl(q_0\|q_1)},
    \qquad\text{hence}\qquad
    \E_0\Big[\sum_{i=1}^{M}N_i\Big]\ \ge\ M\,\frac{\kl(1-\beta\|\delta)}{\kl(q_0\|q_1)} .
\end{equation}
\end{theorem}

\emph{Proof sketch.} The all-certified event has probability at least $1-\beta$ under $P_0$ and at most $\delta$ under $P_i$. Only labels queried at coordinate $i$ distinguish the two laws, so the chain rule for a predictably sampled transcript gives $\KL(P_0^{\mathrm{tr}}\|P_i^{\mathrm{tr}})=\E_0[N_i]\kl(q_0\|q_1)$, and data processing through the indicator of that event yields the bound.

The cost is charged under $P_0$, where nothing changes: guarding against a possible change costs labels even when none occurs. A sequential likelihood-ratio certifier attains the bound up to an overshoot constant, and a deterministic cap pays an extra $\log M$ (Appendix~\ref{app:matching}).

\begin{proposition}[No label-free validation of transport]\label{prop:nolabelfree}
Let $\Psi$ be any transport-validity test measurable with respect to the cheap transcript. Then $P_0(\Psi=1)=P_i(\Psi=1)$ for every coordinate $i$.
\end{proposition}

Residual stability therefore cannot be verified from cheap data and must be assumed; Theorem~\ref{thm:trustedtransport} prices that assumption. The failure is not only theoretical: in an exploratory Folktables stress test with hidden injected failures, a certifier that trusts a calibrated bridge without verifying it false-certifies up to $75\%$ of the unsafe blocks, exactly as often as a no-label oracle, while fresh-label baselines catch every one (Appendix~\ref{app:hiddenfailure}).

\begin{theorem}[Trusted-transport certification region]\label{thm:trustedtransport}
For each candidate $k$ let $H_k=\{\mu_E^{\mathrm{hist}}(k)\le U_H(k)\}$ and $B_k=\{\mu_E^{\mathrm{tar}}(k)\le\mu_E^{\mathrm{hist}}(k)+\rho_k\}$, with $\Pr(\bigcap_kH_k)\ge1-\delta_H$ and $\Pr(\bigcap_kB_k)\ge1-\delta_T$, and define the zero-label slack $s_k=\tau_k-\overline Q_{\mathrm{tar}}(k)-U_H(k)$ and the certifiable set $\mathcal C(\rho)=\{k:\ s_k\ge\rho_k\}$.
\begin{enumerate}
    \item[(i)] Certifying every $k\in\mathcal C(\rho)$ with zero fresh labels certifies some unsafe candidate with probability at most $\delta_H+\delta_T$. No independence between the historical and bridge events is required.
    \item[(ii)] If the bridge is part of the model class deterministically ($\delta_T=0$), the error is at most $\delta_H$, and it is zero if $U_H$ is also exact.
    \item[(iii)] For a common radius $\rho_k\equiv\rho$, $\mathcal C(\rho)$ is nonincreasing in $\rho$, the zero-label certifiable fraction $|\mathcal C(\rho)|/K$ is the empirical survival function of the slacks $\{s_k\}$, and $\mathcal C(\rho)=\emptyset$ once $\rho>\max_ks_k$; the limit $\rho\to\infty$ recovers \Mtwo, in which no candidate is certifiable without current labels.
\end{enumerate}
\end{theorem}

The zero-label benefit of \Mone{} also cannot be learned adaptively from the same transcript (Corollary~\ref{cor:nofreebridge}, Appendix~\ref{app:deferred}).

\section{Portfolio vigilance}\label{sec:method}

We now turn to \Mtwo. Portfolio vigilance tests each candidate by betting. It holds a wealth that starts at one, multiplies it after each purchased label by a factor whose conditional mean is at most one when the candidate is unsafe, and certifies once the wealth is large. Such a wealth process is an e-process, and Ville's inequality \citep{ville1939} bounds the probability that it ever exceeds $1/\delta$ under the null. Formally, the certifier purchases rows of the target pool in one uniformly random order $\pi_1,\pi_2,\dots$ without replacement, and before draw $t$ its information $\mathcal F_{t-1}$ contains the cheap transcript of the whole pool, the fixed historical advice, and the previously purchased outcomes only. For candidate $k$ let $S_{t-1}(k)=\sum_{s<t}Y_{\pi_s}(k)$ be the cumulative purchased trusted loss. Under the certification null $R(k)\ge\tau_k$, the unpurchased rows have mean at least
\begin{equation}\label{eq:mt}
    m_t(k)=\frac{N\tau_k-S_{t-1}(k)}{N-t+1},
\end{equation}
so for any predictable stake $\lambda_t\ge0$ with $1+\lambda_t(m_t-y)\ge0$ on $[0,1]$, the factor $1+\lambda_t(k)\{m_t(k)-Y_{\pi_t}(k)\}$ has conditional mean at most one.

\paragraph{Two experts.} The \emph{ledger expert} forecasts the remaining-pool mean with the historical correction,
\[
    \hat p^{\mathrm L}(k)=\operatorname{clip}\big(\bar Q(k)+\bar A(k),0,1\big),
\]
where $\bar Q(k)$ is the all-pool cheap mean. The \emph{robust expert} starts with no residual correction and learns only from purchased target labels,
\[
    \hat p_t^{\mathrm R}(k)=\operatorname{clip}\Big(\bar Q(k)+\frac{1}{t-1}\sum_{s<t}\big[Y_{\pi_s}(k)-Q_{\pi_s}(k)\big],0,1\Big).
\]
The ledger expert thus trusts history from the first label, whereas the robust expert must learn the error before it bets confidently. For expert $e$ the certification-direction Kelly stake is $\lambda_t^{e,-}(k)=[m_t(k)-\hat p_t^e(k)]_+/\{m_t(k)(1-m_t(k))\}$, with the symmetric rejection-direction bet against the null $R(k)\le\tau_k$. The deployed evidence is the fixed mixture
\begin{equation}\label{eq:mixture}
    E_t^{-}(k)=\tfrac12E_t^{\mathrm L,-}(k)+\tfrac12E_t^{\mathrm R,-}(k),\qquad
    E_t^{+}(k)=\tfrac12E_t^{\mathrm L,+}(k)+\tfrac12E_t^{\mathrm R,+}(k).
\end{equation}
Algorithm~\ref{alg:portfolio} in Appendix~\ref{app:deferred} gives the procedure: a candidate is certified when $E_t^-(k)\ge K/\delta$, rejected when $E_t^+(k)\ge K/\beta$, resolved earlier by exact finite-population bounds, and otherwise left unresolved at the budget, never silently safe.

\begin{theorem}[Anytime validity under arbitrary advice]\label{thm:validity}
With losses in $[0,1]$, uniform sampling without replacement, thresholds and advice fixed without target-role outcomes, and predictable stakes, portfolio vigilance certifies some unsafe candidate with probability at most $\delta$ and rejects some safe candidate with probability at most $\beta$. The guarantee holds for arbitrary historical advice, candidate-specific stopping, shared purchased outcomes, and abstention at budget exhaustion.
\end{theorem}

\emph{Proof sketch.} Under the null, every predictable betting factor has conditional mean at most one, so each expert's wealth and their mixture are nonnegative supermartingales, and Ville's inequality with a union bound over candidates gives the levels; advice enters only through $\mathcal F_{t-1}$. Each assumption is necessary: nonuniform sampling or a stake that looks at the next row can give a betting factor mean above one (Appendix~\ref{app:theoryremarks}).

\begin{theorem}[Optimal fixed-mixture hedge]\label{thm:bestexpert}
For component e-processes $E_t^{\mathrm L},E_t^{\mathrm R}$ and fixed $w\in(0,1)$, $E_t^{\mathrm{mix}}=wE_t^{\mathrm L}+(1-w)E_t^{\mathrm R}$ is an e-process, and on every path
\[
    \log E_t^{\mathrm{mix}}\ \ge\ \max\big\{\log E_t^{\mathrm L}+\log w,\ \log E_t^{\mathrm R}+\log(1-w)\big\}.
\]
The worst-case log-evidence regret to the better component is $c(w)=\max\{-\log w,-\log(1-w)\}$, no smaller uniform constant is possible for that $w$, and $w=1/2$ uniquely minimizes $c(w)$ at $\log 2$.
\end{theorem}

In words, the mixture never trails the better expert by more than a factor of two in evidence; in label units, it reaches any evidence level no later than the first expert reaches twice that level (Corollary~\ref{cor:envelope}, Appendix~\ref{app:deferred}).

\subsection{What history can and cannot buy: a canonical phase diagram}\label{sec:phase}

What does it cost \emph{not to know} whether historical residual information is still accurate? A canonical model answers this exactly. Trusted losses of the target block are i.i.d.\ $\mathrm{Ber}(p)$ with $p<m$, where $m=\tau$ is the certification boundary; this is the infinite-pool limit of Eq.~\ref{eq:mt}. The proxy has constant mean $\bar Q$, so $p=\bar Q+\mu_E^{\mathrm{tar}}$, and the ledger supplies $\mu_E^{\mathrm{hist}}$. The \emph{advice mismatch}
\begin{equation}\label{eq:eta}
    \eta=\mu_E^{\mathrm{tar}}-\mu_E^{\mathrm{hist}}
\end{equation}
turns the ledger forecast into $r=\bar Q+\mu_E^{\mathrm{hist}}=p-\eta$. Positive $\eta$ is optimistic stale advice (history understates current risk) and negative $\eta$ is pessimistic stale advice. The robust expert forecasts with the running mean of purchased losses, truncated below at a constant $\varepsilon\in(0,p]$ for the analysis.

\begin{proposition}[Transport-vigilance phase diagram]\label{prop:phase}
In the canonical model, almost surely:
\begin{enumerate}
    \item[(i)] the ledger expert's log evidence grows at rate $\lim_t t^{-1}\log E_t^{\mathrm L}=G_L(\eta)$, where
    \[
        G_L(\eta)=\begin{cases}
        0, & \eta\le\eta_-:=p-m \quad(\text{forecast at or above the boundary: no bet}),\\
        \kl(p\|m)-\kl(p\|p-\eta), & \eta_-<\eta<p,
        \end{cases}
    \]
    and $E_t^{\mathrm L}\to0$ if $\eta\ge p$ (a forecast clipped at zero is ruined by the first unit loss);
    \item[(ii)] $G_L$ is uniquely maximized at $\eta=0$ with value $\kl(p\|m)$, is positive exactly on $(\eta_-,\eta_+)$, where $\eta_+=p-r_\star\in(0,p)$ and $r_\star\in(0,p)$ is the unique solution of $\kl(p\|r_\star)=\kl(p\|m)$, and is negative for $\eta>\eta_+$;
    \item[(iii)] the robust expert's rate is $\lim_t t^{-1}\log E_t^{\mathrm R}=\kl(p\|m)$ for every $\eta$;
    \item[(iv)] the portfolio's rate is $\lim_t t^{-1}\log E_t^{\mathrm{mix}}=\max\{G_L(\eta),\kl(p\|m)\}=\kl(p\|m)$.
\end{enumerate}
\end{proposition}

History therefore shortens a decision but never raises the evidence rate; pessimistic stale advice costs the ledger only power, optimistic stale advice beyond $\eta_+$ makes its evidence decay, and the portfolio keeps the vigilance rate for every $\eta$ (Figure~\ref{fig:phase}). Item (i) follows from an exact one-step identity for the ledger's log growth (Theorem~\ref{thm:mismatch} in Appendix~\ref{app:theoryremarks}). The price of not knowing $\eta$ is explicit: the portfolio needs at most $\log2/G_L(\eta)$ more expected labels than the ledger's bound, and at $\eta=0$ its bound matches the lower bound of Theorem~\ref{thm:vigilance} to first order in $\log(1/\delta)$ (Corollary~\ref{cor:hedgeprice}, Appendix~\ref{app:deferred}).

\begin{figure}[t]
\centering
\includegraphics[width=\textwidth]{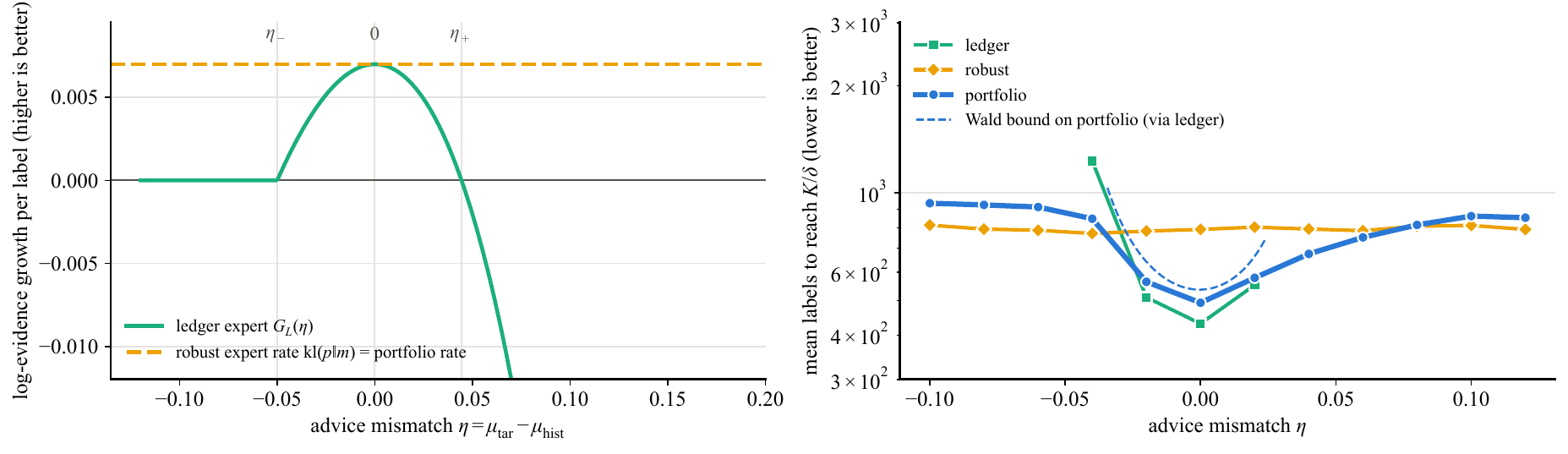}
\caption{Canonical phase diagram ($p=0.20$, $m=0.25$, $K=1$, $\delta=0.05$). \textbf{Left (exact):} ledger growth $G_L(\eta)$ against the flat robust and portfolio rate $\kl(p\|m)$. \textbf{Right (simulation, 2{,}000 paths per point):} mean labels to reach $K/\delta$; ledger points are omitted where over 1\% of paths fail within 6{,}000 labels, and the dashed curve is the bound of Corollary~\ref{cor:hedgeprice}.}
\label{fig:phase}
\end{figure}

\section{Experimental design}\label{sec:experiments}

\paragraph{Protocol.} To keep analysis choices independent of the results, we fix everything before analyzing the final roles: five random sampling orders (seeds), block-level $\delta=0.05$ and $\beta=0.10$, budgets from $1\%$ to $100\%$ of each block, thresholds of ledger risk plus $0.05$, and the comparators. All methods share pools, orders, and budgets, and development uses separate datasets.

\paragraph{Data.} DICES-990 contains 72{,}103 safety ratings of 990 conversations \citep{aroyo2023dices}; we audit four outcome-level candidates (harmful content, unfair bias, misinformation, and overall safety) on 21{,}702 held-out rating rows in 12 locale-by-hash blocks, with a term-frequency inverse-document-frequency (TF-IDF) logistic proxy. CIFAR-10N provides three human annotations per CIFAR-10 training image \citep{wei2022cifarn}; the loss is disagreement with the canonical class, and 14{,}767 held-out images form six blocks with one candidate each. Both are split by hash into proxy-training, ledger, calibration, and final roles (Appendix~\ref{app:algorithms}), and each final-role block is one target pool.

\paragraph{Comparators and statistics.} We compare the portfolio with its two components (ledger-only and robust-only), fresh without-replacement (WoR) betting, a selective residual confidence sequence (CS) \citep{ao2026bestarm}, passive fixed-budget sampling, and a PPRM-style conjugate-mixture empirical-Bernstein (CM-EB) monitor. The last keeps the prediction-plus-residual estimator of \citet{zhang2026pprm} but is adapted to our fixed-pool estimand; we do not claim it reproduces the original procedure exactly (Appendix~\ref{app:algorithms}). The primary metric is the paired geometric-mean portfolio/PPRM label ratio at full budget, with 95\% hierarchical bootstrap intervals over blocks and then seeds, and the area under the correct-resolution-versus-budget curve (AUC). In tables \up{} means higher is better and \down{} lower is better.

\paragraph{Post-confirmatory repeated-judge study.} Designed after the primary protocol was fixed, this study replaces the cheap proxy with a legacy rubric (Judge A) and an updated contextual-risk rubric (Judge B) of one fixed LLM, so it isolates evaluator-configuration drift rather than a model update. Each rubric judges every item three times, on all 990 DICES conversations and all 5{,}654 human-annotated ToxicChat-0124 examples \citep{lin2023toxicchat}; Judge A forms the ledger and Judge B the current proxy in $A\!\to\!B$. The analysis uses the five prespecified primary orders rather than the 50 planned in the written post-confirmatory protocol, averages them within block before $10{,}000$ block-bootstrap replicates, and never treats them as independent units (Appendix~\ref{app:llmshift}).

\section{Results}\label{sec:results}

\subsection{Held-out natural data}\label{sec:primary}

Table~\ref{tab:primary} gives the full-budget outcomes. The paired portfolio/PPRM label ratio is $0.465$ (95\% CI $[0.327,0.575]$) on CIFAR-10N and $0.740$ ($[0.618,0.877]$) on DICES-990, a reduction of $53.5\%$ and $26.0\%$. The ledger-only expert is slightly cheaper because these data favor historical advice; the portfolio pays a small hedge for not knowing this, and beats every other deployable baseline. When history is stale the order reverses: after an LLM judge's rubric changes, the portfolio needs fewer labels than ledger-only on DICES-990 and ToxicChat (Section~\ref{sec:sensitivity}), and under biased advice ledger-only needs up to $1.66$ times the labels of the cheaper component (Section~\ref{sec:hedge}). The portfolio makes no false certification; its one DICES false rejection, shared by the robust-only expert, is a safe candidate rejected early in one block. Zero observed false certifications are an empirical observation, not evidence of zero error probability; validity follows from the stated assumptions and procedure (Theorem~\ref{thm:validity}).

\begin{table}[t]
\caption{Full-budget results on the held-out final roles. Ledger-only and robust-only are the portfolio's components. FC/FR: false certifications/rejections over 30 (CIFAR-10N) and 240 (DICES-990) decisions; desc.: descriptive. Labels are mean $\pm$ standard deviation (SD) over block-seed pairs; comparisons use the paired ratios in the text. Bold marks the lowest labels and highest AUC within each dataset. Selective residual CS buys labels on the passive schedule but resolves nothing before full budget, so its AUC is only the last budget interval, $0.100$.}
\label{tab:primary}
\small
\setlength{\tabcolsep}{4pt}
\begin{center}
\begin{tabular}{llrrrrrr}
\toprule
Dataset & Method & Labels \down & Resolution \up & Release (desc.) & FC \down & FR \down & AUC \up \\
\midrule
CIFAR-10N & \textbf{Portfolio vigilance} & $279.9\pm171.0$ & $100.00\%$ & $100.00\%$ & $0$ & $0$ & $0.880$ \\
& Ledger-only (component) & $\mathbf{248.9}\pm166.5$ & $100.00\%$ & $100.00\%$ & $0$ & $0$ & $\mathbf{0.894}$ \\
& Robust-only (component) & $967.8\pm675.5$ & $100.00\%$ & $100.00\%$ & $0$ & $0$ & $0.602$ \\
& PPRM-style CM-EB & $574.6\pm325.7$ & $100.00\%$ & $100.00\%$ & $0$ & $0$ & $0.743$ \\
& Fresh WoR betting & $1{,}492.6\pm438.0$ & $100.00\%$ & $100.00\%$ & $0$ & $0$ & $0.400$ \\
& Passive fixed budget & $2{,}461.2\pm39.0$ & $100.00\%$ & $100.00\%$ & $0$ & $0$ & $0.840$ \\
& Selective residual CS & $2{,}461.2\pm39.0$ & $100.00\%$ & $100.00\%$ & $0$ & $0$ & $0.100$ \\
\midrule
DICES-990 & \textbf{Portfolio vigilance} & $1{,}182.2\pm678.2$ & $99.58\%$ & $72.50\%$ & $0$ & $1$ & $\mathbf{0.569}$ \\
& Ledger-only (component) & $\mathbf{1{,}181.4}\pm689.3$ & $100.00\%$ & $72.92\%$ & $0$ & $0$ & $0.534$ \\
& Robust-only (component) & $1{,}368.8\pm527.8$ & $99.58\%$ & $72.50\%$ & $0$ & $1$ & $0.501$ \\
& PPRM-style CM-EB & $1{,}398.3\pm542.0$ & $100.00\%$ & $72.92\%$ & $0$ & $0$ & $0.402$ \\
& Fresh WoR betting & $1{,}586.1\pm295.1$ & $99.17\%$ & $72.92\%$ & $1$ & $1$ & $0.316$ \\
& Passive fixed budget & $1{,}808.5\pm176.6$ & $100.00\%$ & $72.92\%$ & $0$ & $0$ & $0.521$ \\
& Selective residual CS & $1{,}808.5\pm176.6$ & $100.00\%$ & $72.92\%$ & $0$ & $0$ & $0.100$ \\
\bottomrule
\end{tabular}
\end{center}
\end{table}

The advantage holds across the budget frontier (Figure~\ref{fig:frontier} in Appendix~\ref{app:primaryextra}): resolution AUC is $0.880$ against $0.743$ on CIFAR-10N and $0.569$ against $0.402$ on DICES-990, and at 20\% of the budget the portfolio resolves $90\%$ of CIFAR-10N decisions against $40\%$ and $34.58\%$ of DICES-990 decisions against $13.75\%$. The DICES gain varies across blocks, from $0.422$ to $0.999$ (Appendix~\ref{app:primaryextra}).

\subsection{External blocks and corrupted advice}\label{sec:hedge}

On six external blocks from CIFAR-10 with controlled noise, SVHN \citep{krizhevsky2009cifar,netzer2011svhn}, and UCI Naval Propulsion \citep{coraddu2016naval}, the unchanged portfolio uses fewer labels than the PPRM-style monitor on all six (ratios $0.296$ to $0.916$), with no false certification in any study (Appendix~\ref{app:primaryextra}).

Because the natural data favor accurate history, the hedge is best tested where history is wrong. We therefore bias the historical advice on synthetic finite pools, holding everything else fixed. With exact advice the ledger-only, portfolio, and robust-only experts use $353.4$, $381.7$, and $594.5$ labels; under a $-0.15$ bias they use $986.7$, $624.5$, and $594.5$. Across seven biases the portfolio stays within $8.0\%$ of the cheaper component, while the ledger-only expert reaches $1.66$ times it (Appendix~\ref{app:stress}), which is the finite-pool counterpart of Figure~\ref{fig:phase}.

\subsection{Sensitivity and evaluator-shift experiments}\label{sec:sensitivity}

\paragraph{Sensitivity.} The advantage does not hinge on the prespecified threshold: margins from $0.02$ to $0.10$ keep the ratio between $0.403$ and $0.465$ on CIFAR-10N and between $0.649$ and $0.740$ on DICES-990 (Appendix~\ref{app:naturalsensitivity}).

\paragraph{LLM evaluator shift.} The post-confirmatory repeated-judge study (Section~\ref{sec:experiments}) tests the motivating mechanism directly. First, the rubric change is real rather than judge noise: the mean absolute score change between rubrics exceeds that within a rubric by $0.0476$ (95\% CI $[0.0331,0.0633]$) on DICES-990 final items and by $0.0379$ ($[0.0301,0.0460]$) on ToxicChat. Second, after the change (Table~\ref{tab:repeatedshift}) the portfolio needs $0.790$ ($[0.667,0.909]$) and $0.631$ ($[0.520,0.770]$) times the labels of the PPRM-style monitor, with no false certification. Third, unlike on the stable primary data, stale history now makes the ledger expert the costlier one: the portfolio needs $140.7$ fewer labels than ledger-only on DICES-990 ($[-256.2,-46.2]$ for the difference) and $12.7$ fewer on ToxicChat ($[-20.8,-5.9]$), and it is cheaper than ledger-only in all nine pairings of an A run with a B run on both datasets. This is the hedge that \Mtwo{} predicts. The portfolio also beats robust-only on DICES-990 ($-109.1$, $[-199.3,-32.6]$); on ToxicChat the two tie ($137.6$ against $137.5$ labels; difference $+0.05$, $[-4.65,4.05]$), so we claim no advantage there. Charging the full block for every error or unresolved decision leaves these conclusions unchanged. For comparison, the calibration-based endpoint test Noisy-but-Valid \citep{feng2026noisy} makes no false certification but also certifies none of the eight safe endpoint cases (Appendix~\ref{app:llmshift}). On DICES-990 the ratio to PPRM is no smaller than in the primary study ($0.740$): the shift shows that the hedge protects the portfolio, not that it widens its lead over PPRM. An initial single-run study on DICES-990 gives the same pattern (ratio $0.779$, $[0.629,0.924]$; Appendix~\ref{app:llmshift}).

\begin{table}[t]
\caption{Repeated-judge study under the shifted $A\!\to\!B$ condition at full budget. Labels are means over blocks and the five random orders. P/P: portfolio/PPRM-style label ratio with its block-bootstrap 95\% interval. Shift excess: between-rubric minus within-rubric mean absolute judge-score change on final items, with its item-bootstrap 95\% interval. The portfolio makes no false certification on either dataset.}
\label{tab:repeatedshift}
\small
\setlength{\tabcolsep}{3pt}
\begin{center}
\begin{tabular}{lrrrrll}
\toprule
Dataset & Portfolio \down & Ledger-only \down & Robust-only \down & PPRM \down & P/P ratio \down & Shift excess \\
\midrule
DICES-990 & $\mathbf{1{,}305.2}$ & $1{,}445.9$ & $1{,}414.2$ & $1{,}544.3$ & $0.790$ $[0.667,0.909]$ & $0.0476$ $[0.0331,0.0633]$ \\
ToxicChat & $137.6$ & $150.2$ & $\mathbf{137.5}$ & $208.5$ & $0.631$ $[0.520,0.770]$ & $0.0379$ $[0.0301,0.0460]$ \\
\bottomrule
\end{tabular}
\end{center}
\end{table}

\paragraph{The price of trusted transport.} The repeated-judge data also illustrate \Mone{} (Appendix~\ref{app:rho}): at $\rho=0$ a zero-label certificate would cover $81\%$ of DICES-990 cells under rubric A although the bridge holds for only $35\%$, and the smallest radius at which the bridge holds everywhere leaves at most $2\%$ certifiable, so a bridge wide enough to be true removes the zero-label advantage.

\section{Limitations}\label{sec:limitations}

The lower bound concerns i.i.d.\ Bernoulli draws rather than finite pools sampled without replacement, and in a scaling study the portfolio uses four to five times the bound (Appendix~\ref{app:stress}); Proposition~\ref{prop:phase} and Corollary~\ref{cor:hedgeprice} hold only in the canonical model. Validity rests on the declared design of bounded losses, prespecified thresholds, and uniform sampling: if audited rows are chosen by convenience, a false certificate can release a harmful update whose cost falls on its users. Error control is block-level. On the empirical side, the PPRM-style monitor is our adaptation of the original method. The LLM studies change the rubric of one fixed model, not the model itself, and were designed after the primary protocol was fixed; the repeated-judge study uses the five primary random orders rather than the 50 its written protocol planned, and one of its six comparisons (portfolio against robust-only on ToxicChat) is a tie (Appendix~\ref{app:llmshift}). The primary label counts do not penalize incorrect decisions, although a penalized cost leaves the repeated-judge conclusions unchanged.

\section{Conclusion}\label{sec:conclusion}

Historical evaluator evidence has two statistically different roles. If a trusted transport condition connects an old audit to the present, the audit may constrain current risk and replace fresh labels, at an error cost that should be declared; if not, no label-free test can tell whether the connection still holds, and honest certification must keep buying current evidence at a rate we characterize. In the second regime history can still shorten a decision, but it cannot raise the rate at which evidence accumulates, and stale history can destroy it. Portfolio vigilance pays a small hedge when history is reliable and avoids the larger loss when it is stale: it needs markedly fewer trusted labels than a matched prediction-powered monitor and, when an LLM judge's rubric changes, fewer than trusting history alone, consistent with the \Mtwo{} mechanism. A release process should therefore state, for every certificate, whether historical evidence entered its validity or only its efficiency. Open questions are how to certify transport itself with few labels when the bridge is only partially testable, and the price of vigilance for finite pools sampled without replacement.

\subsection*{AI use statement}

The authors use generative AI tools to assist with literature search and summarization, methodological critique, software design, code drafting and debugging, interpretation of statistical results, figure and table preparation, and manuscript drafting and editing. The authors check every AI-assisted mathematical and empirical claim against the code, the stored results, the primary literature, and reproducible analyses. The authors take responsibility for the final text, proofs, code, and reported results.

\subsection*{Ethics statement}

All experiments use public datasets under their published terms. DICES contains sensitive conversational safety content and socially situated human judgments \citep{aroyo2023dices}, CIFAR-10N contains crowd annotations \citep{wei2022cifarn}, and ToxicChat contains real user prompts, some of them toxic, with human toxicity labels \citep{lin2023toxicchat}. We report only declared statistical losses and do not treat demographic disagreement as a single objective notion of safety; the locale heterogeneity in Section~\ref{sec:primary} is descriptive. No model or threshold in this paper should be read as a recommendation about people. Results are benchmark measurements, not safety, clinical, or fairness guarantees. The study collects no new human-subject data.

\subsection*{Reproducibility statement}

The appendix contains complete proofs, the full per-method and per-block result tables, and the design of every study, including the data roles, random orders, budgets, thresholds, and comparators. The primary protocol is fixed before any analysis of the final roles, and every analysis designed afterwards is marked as additional or post-confirmatory. The supplementary material contains the method and baseline code, the deterministic role construction, the frozen LLM-judge outputs, the per-block results, and scripts that regenerate every table and figure. All experiments run on CPUs.

\bibliography{iclr2027_conference}
\bibliographystyle{iclr2027_conference}

\appendix

The appendix follows the order of the paper. Appendix~\ref{app:notation} summarizes the notation and Appendix~\ref{app:related} extends the related work. Appendix~\ref{app:deferred} gives the pseudocode and the corollaries stated in brief in the main text, Appendix~\ref{app:proofs} proves every result of Sections~\ref{sec:vigilance} and~\ref{sec:method}, Appendix~\ref{app:matching} gives the matching sequential certifier and the cost of deterministic label caps, and Appendix~\ref{app:theoryremarks} discusses the scope of the theory. Appendix~\ref{app:algorithms} describes the data roles and the comparator. Appendices~\ref{app:primaryextra} to~\ref{app:llmshift} report additional results in the order of Section~\ref{sec:results}: per-block and external results, corrupted advice and scaling, natural-data sensitivity, and the single-run and repeated-judge LLM evaluator-configuration studies. Appendix~\ref{app:hiddenfailure} reports the Folktables hidden-failure study of unverified transport cited in Section~\ref{sec:vigilance}.

\section{Notation}\label{app:notation}

Table~\ref{tab:notation} summarizes the notation used throughout the paper, grouped in the order in which the symbols appear.

\begin{table}[h]
\caption{Notation.}
\label{tab:notation}
\small
\begin{center}
\begin{tabular}{p{1.45in}p{3.9in}}
\toprule
Symbol & Meaning \\
\midrule
\Mzero, \Mone, \Mtwo & Hidden residual drift; trusted transport; untrusted advice \\
$N$, $K$ & Rows in the finite target pool; number of prespecified candidate claims \\
$Y_i(k)$, $Q_i(k)$ & Trusted loss and cheap prediction of it for row $i$ and candidate $k$, both in $[0,1]$ \\
$E_i(k)$ & Evaluator residual $Y_i(k)-Q_i(k)$ \\
$R(k)$, $\tau_k$ & Current finite-population risk of candidate $k$; its release threshold \\
$\delta$, $\beta$ & False-certification and false-rejection levels \\
$\mu_E^{\mathrm{hist}}(k)$, $\mu_E^{\mathrm{tar}}(k)$ & Historical and current (target) mean residual \\
$\bar A(k)$ & Historical residual advice supplied by the ledger \\
$\bar Q(k)$, $\overline Q_{\mathrm{tar}}(k)$ & All-pool mean of the cheap prediction \\
$U_H(k)$, $\delta_H$ & Historical upper bound on $\mu_E^{\mathrm{hist}}(k)$ and its error probability \\
$\rho_k$, $\delta_T$ & Trusted-transport (bridge) radius and its error probability \\
$s_k$, $\mathcal C(\rho)$ & Zero-label slack $\tau_k-\overline Q_{\mathrm{tar}}(k)-U_H(k)$; candidates certifiable with zero labels \\
$\Psi$ & Any test measurable with respect to the cheap transcript \\
$M$ & Monitored coordinates in the lower-bound model (e.g.\ blocks times strata) \\
$q_0$, $q_1$ & Safe and unsafe Bernoulli trusted-loss means in the lower-bound submodel \\
$P_0$, $P_i$ & Stationary safe law; law with coordinate $i$ made unsafe \\
$C_i$, $N_i$ & Certification indicator and number of labels queried at coordinate $i$ \\
$\kl(a\|b)$ & Binary relative entropy \\
$\pi_t$, $\mathcal F_{t-1}$ & Row purchased at step $t$; information available before draw $t$ \\
$S_{t-1}(k)$, $m_t(k)$ & Cumulative purchased trusted loss; null-boundary mean of the unpurchased rows \\
$\hat p^{\mathrm L}(k)$, $\hat p_t^{\mathrm R}(k)$ & Ledger-expert and robust-expert forecasts \\
$\lambda_t^{e,\pm}(k)$ & Kelly stake of expert $e$ in the certification ($-$) or rejection ($+$) direction \\
$E_t^{\mathrm L,\pm}$, $E_t^{\mathrm R,\pm}$ & Wealth (e-process) of the ledger and robust experts \\
$E_t^{\pm}$, $E_t^{\mathrm{mix}}$ & Portfolio evidence, the fixed mixture of the two experts \\
$w$, $c(w)$ & Ledger mixture weight; worst-case log-evidence regret to the better expert \\
$A$, $T_e(A)$ & Evidence level; first purchased-label time at which expert $e$ reaches $A$ \\
$p$, $m$ & Current trusted-loss mean and certification boundary in the canonical model \\
$r$, $B_r(Y)$ & Ledger forecast $p-\eta$; one-step Kelly factor with forecast $r$ \\
$\eta$ & Advice mismatch $\mu_E^{\mathrm{tar}}-\mu_E^{\mathrm{hist}}$ \\
$\eta_-$, $\eta_+$, $r_\star$ & Phase boundaries $p-m$ and $p-r_\star$; solution of $\kl(p\|r_\star)=\kl(p\|m)$ below $p$ \\
$G_L(\eta)$ & Ledger expert's expected log-evidence growth per label \\
$\varepsilon$, $c_\eta$ & Truncation of the robust forecast; largest ledger log-wealth increment \\
\bottomrule
\end{tabular}
\end{center}
\end{table}

\section{Extended related work}\label{app:related}

Table~\ref{tab:positioning} compares the paper with the closest prior results at the level of their guarantees. Prediction-powered e-values give prediction-assisted counterparts of e-value procedures \citep{csillag2025ppevalues}, and simulator-guided betting builds anytime-valid certificates whose validity holds for any simulator bank while trust shifts across a portfolio of simulators \citep{chen2026sim2real}, so using untrusted side information, or a portfolio over its sources, to choose bets is not new. Noisy-but-Valid estimates a current judge's true- and false-positive rates from human calibration labels and propagates that uncertainty into a valid test \citep{feng2026noisy}; its problem is judge imperfection, whereas ours is whether residual evidence from a historical judge or population still describes the current one. Disagreement-based auditing proves matching label complexity from observable model disagreement \citep{balachandran2026disagreement}, whereas the hidden residual changes we study may leave no trace in any cheap observable.

\begin{table}[t]
\caption{Positioning against the closest prior results; guarantees are not numerically comparable.}
\label{tab:positioning}
\small
\begin{center}
\begin{tabular}{p{1.3in}p{1.85in}p{1.95in}}
\toprule
Work & Guarantee & What this paper adds \\
\midrule
\citet{zhang2026pprm} & Anytime monitoring of prediction-powered running-average risk under shift. & Current-risk release decisions and the status of a historical residual ledger. \\
\citet{feng2026noisy} & Valid endpoint tests with an imperfect current judge. & Whether historical residuals may enter validity after the judge changes. \\
\citet{chen2026sim2real} & Anytime-valid confidence sequence for a mean from i.i.d.\ real outcomes, valid for any simulator bank; exponential-weights trust combines simulator forecasts into one bet, with wealth-regret bounds against the best fixed simulator. & A fixed mixture of the wealths of a history-guided and a label-only expert, within $\log2$ evidence of the better one (minimax constant), a stale-advice growth identity, and finite-pool release decisions. \\
\citet{xu2024active} & Side information guides betting without entering validity. & A two-expert fixed mixture with an optimal $\log2$ hedge and a stale-advice growth identity. \\
\citet{csillag2025ppevalues} & Prediction-powered counterparts of e-value procedures. & Historical, not current, prediction residuals, and when they may enter validity. \\
\citet{balachandran2026disagreement} & Matching label complexity for update auditing from observable disagreement. & A rate driven by residual transport and its untestable part, not by disagreement. \\
\citet{mazzetto2023drift} & Latest-distribution estimation under unknown drift. & Selective label counting, a release certificate, and a cost charged on a stationary path. \\
\citet{low1997}; \citet{cailow2004} & Non-adaptivity of honest confidence intervals. & A sequential selective-label vigilance rate and a transport condition that escapes it. \\
\citet{doula2026transfer}; \citet{klivans2024testable} & Transfer or learning under shift with a surrogate or testable condition. & Impossibility of label-free validation of the residual-transport contract. \\
\bottomrule
\end{tabular}
\end{center}
\end{table}

\section{Algorithm and deferred statements}\label{app:deferred}

This appendix collects the pseudocode and three corollaries stated in brief in the main text; their proofs are in Appendix~\ref{app:proofs}.

\begin{algorithm}[h]
\caption{Two-sided portfolio vigilance for one finite target block}
\label{alg:portfolio}
\small
\begin{algorithmic}[1]
\REQUIRE pool of $N$ rows with cheap scores $Q_i(k)$; thresholds $\tau_k$; fixed advice $\bar A(k)$; levels $\delta,\beta$; label budget $B$
\ENSURE a decision in \{certify, reject, abstain\} for each of the $K$ candidates
\STATE Draw a uniform random order $\pi$ of the pool; set all expert wealths to $1$
\FOR{$t=1,\dots,B$}
    \IF{every candidate is resolved} \STATE \textbf{stop} \ENDIF
    \FOR{each unresolved candidate $k$}
        \STATE compute $m_t(k)$ by Eq.~\ref{eq:mt}, forecasts $\hat p^{\mathrm L}(k)$ and $\hat p_t^{\mathrm R}(k)$ from $\mathcal F_{t-1}$, and the four directional Kelly stakes \COMMENT{predictable: never uses row $\pi_t$}
    \ENDFOR
    \STATE purchase $Y_{\pi_t}(k)$ for every candidate defined on row $\pi_t$ \COMMENT{one label updates all candidates}
    \STATE update the four expert wealths of each unresolved candidate and form $E_t^{\pm}(k)$ by Eq.~\ref{eq:mixture}
    \STATE certify $k$ if $E_t^-(k)\ge K/\delta$ or $(S_t+N-t)/N\le\tau_k$; reject $k$ if $E_t^+(k)\ge K/\beta$ or $S_t/N>\tau_k$
\ENDFOR
\STATE report unresolved candidates as abstentions
\end{algorithmic}
\end{algorithm}

\begin{corollary}[No free adaptation to an unverified bridge]\label{cor:nofreebridge}
If a safe world and an unsafe \Mzero{} world induce the same pre-label transcript, any procedure that is uniformly $\delta$-valid over \Mzero{} and decides from that transcript to issue a zero-label transport certificate does so in the safe world with probability at most $\delta$.
\end{corollary}

\begin{corollary}[Pathwise stopping envelope]\label{cor:envelope}
Let $T_{\mathrm{mix}}(A)$, $T_{\mathrm L}(A)$, and $T_{\mathrm R}(A)$ be the first purchased-label times at which the mixture, ledger, and robust evidence reach level $A$. On every path $T_{\mathrm{mix}}(A)\le\min\{T_{\mathrm L}(A/w),\,T_{\mathrm R}(A/(1-w))\}$; with $w=1/2$ the portfolio reaches $A$ no later than the first component reaches $2A$.
\end{corollary}

\begin{corollary}[Price of hedging]\label{cor:hedgeprice}
In the canonical model with $\eta\in(\eta_-,\eta_+)$, let $T_{\mathrm{port}}$ be the first label at which the portfolio's certification evidence reaches $K/\delta$. Then
\[
    \E[T_{\mathrm{port}}]\ \le\ \frac{\log(2K/\delta)+c_\eta}{G_L(\eta)},\qquad c_\eta=\log\frac{1-p+\eta}{1-m},
\]
while the ledger expert alone satisfies the same bound with $\log(K/\delta)$ in place of $\log(2K/\delta)$. Not knowing whether history is accurate therefore costs at most $\log2/G_L(\eta)$ additional expected labels relative to the ledger's bound, and at $\eta=0$ the portfolio's bound is $\{\log(2K/\delta)+c_0\}/\kl(p\|m)$, which matches the lower bound $\kl(1-\beta\|\delta)/\kl(q_0\|q_1)$ of Theorem~\ref{thm:vigilance}, taken with $q_0=p$ and $q_1\downarrow m$, to first order in $\log(1/\delta)$ up to the factor $1-\beta$ and additive constants.
\end{corollary}

\section{Proofs}\label{app:proofs}

This section proves the results of the main text in the order in which they appear. Throughout, $P^{\mathrm{tr}}$ denotes the law of the observed transcript, which contains every cheap observation, every purchased label, and every internal decision of the auditor.

\subsection{Proofs for Section~\ref{sec:vigilance}}\label{app:proofs-vigilance}

\paragraph{Proof of Theorem~\ref{thm:vigilance}.}
Let $A=\{C_i=1\text{ for every }i\}$. Liveness gives $P_0(A)\ge1-\beta$. Under $P_i$ coordinate $i$ is unsafe, so certifying it violates honesty and $P_i(A)\le\delta$. Only labels queried at coordinate $i$ have different laws under $P_0$ and $P_i$. If $\E_0[N_i]=\infty$ the claim is immediate. Otherwise apply the chain rule for relative entropy to the transcript stopped after the first $r$ queries at coordinate $i$. Because each query decision is predictable given the past transcript, the per-step contribution is $\kl(q_0\|q_1)$ when a label at coordinate $i$ is drawn and zero otherwise, so $\KL(P_0^{\mathrm{tr}}\|P_i^{\mathrm{tr}})=\E_0[N_i\wedge r]\kl(q_0\|q_1)$ for the stopped transcript. Letting $r\to\infty$ and applying monotone convergence gives $\KL(P_0^{\mathrm{tr}}\|P_i^{\mathrm{tr}})=\E_0[N_i]\kl(q_0\|q_1)$. Data processing through the indicator of $A$ gives
\begin{equation}
    \E_0[N_i]\kl(q_0\|q_1)\ \ge\ \kl\big(P_0(A)\,\|\,P_i(A)\big)\ \ge\ \kl(1-\beta\|\delta),
\end{equation}
using that $\kl(a\|b)$ is increasing in $a$ above $b$ and decreasing in $b$ below $a$, with $1-\beta>\delta$. Summing over $i$ gives the cumulative bound. $\square$

A by-product is worth recording. On the event $N_i=0$ the transcript has the same law under both measures, so $P_0(C_i=1,N_i=0)=P_i(C_i=1,N_i=0)\le\delta$, and subtracting from $P_0(C_i=1)\ge1-\beta$ gives $P_0(N_i\ge1)\ge1-\beta-\delta$. Vigilance is not only expensive in expectation; with high probability it happens at every coordinate.

\paragraph{Proof of Proposition~\ref{prop:nolabelfree}.}
The cheap transcript has the same law under $P_0$ and $P_i$ by construction, and $\Psi$ is measurable with respect to it, so $\Psi$ has the same law under both. $\square$

\paragraph{Proof of Theorem~\ref{thm:trustedtransport}.}
(i) Let $H=\bigcap_kH_k$ and $B=\bigcap_kB_k$. On $H\cap B$, every $k$ satisfies $\mu_E^{\mathrm{tar}}(k)\le\mu_E^{\mathrm{hist}}(k)+\rho_k\le U_H(k)+\rho_k$. For $k\in\mathcal C(\rho)$, $s_k\ge\rho_k$ means $\overline Q_{\mathrm{tar}}(k)+U_H(k)+\rho_k\le\tau_k$, so $R(k)=\overline Q_{\mathrm{tar}}(k)+\mu_E^{\mathrm{tar}}(k)\le\tau_k$ and $k$ is safe. A false zero-label certificate therefore requires $H^c\cup B^c$, and $\Pr(H^c\cup B^c)\le\delta_H+\delta_T$ by a union bound, which needs no independence. The proxy mean $\overline Q_{\mathrm{tar}}(k)$ is computed exactly on the whole pool and contributes no error. (ii) With $\delta_T=0$ the bound reads $\delta_H$; if $U_H$ also holds surely, $H^c\cup B^c$ is null. (iii) $k\in\mathcal C(\rho)$ iff $s_k\ge\rho$, so $\rho\le\rho'$ implies $\mathcal C(\rho')\subseteq\mathcal C(\rho)$, and $|\mathcal C(\rho)|/K=K^{-1}\sum_k\mathbf 1\{s_k\ge\rho\}$ is the empirical survival function of the slacks, which vanishes for $\rho>\max_ks_k$. $\square$

\paragraph{Proof of Corollary~\ref{cor:nofreebridge}.}
Let $\mathcal G_0$ be the pre-label transcript and $A\in\mathcal G_0$ the event that a zero-label certificate is issued. The two worlds restrict to the same law on $\mathcal G_0$, so $P_{\mathrm{safe}}(A)=P_{\mathrm{unsafe}}(A)\le\delta$. Any rule that tries to infer from the same transcript whether to trust a bridge is still $\mathcal G_0$-measurable and is covered by the same equality. $\square$

\subsection{Proofs for Section~\ref{sec:method}}\label{app:proofs-method}

\paragraph{Proof of Theorem~\ref{thm:validity}.}
Fix a candidate and suppress $k$. Under $R\ge\tau$, the rows not yet drawn sum to $NR-S_{t-1}\ge N\tau-S_{t-1}$, so uniform sampling from them gives $\E[Y_{\pi_t}\mid\mathcal F_{t-1}]\ge m_t$ with $m_t$ as in Eq.~\ref{eq:mt}. For predictable $\lambda_t\ge0$ with $1+\lambda_t(m_t-y)\ge0$ for all $y\in[0,1]$, $\E[1+\lambda_t(m_t-Y_{\pi_t})\mid\mathcal F_{t-1}]\le1$, so each expert's wealth is a nonnegative supermartingale with unit start, and so is any fixed convex mixture of the two. Ville's inequality bounds the probability that the certification mixture of a given unsafe candidate ever reaches $K/\delta$ by $\delta/K$, and a union bound over the at most $K$ candidates gives the family-wise false-certification bound $\delta$. The rejection direction reverses the centered factor under $R\le\tau$ and gives $\beta$ in the same way. The stakes use only $\mathcal F_{t-1}$, which contains the historical advice as a fixed quantity; no property of the advice is used, which is why the guarantee holds for arbitrary advice. When $m_t\notin(0,1)$ the implementation skips betting and relies on the deterministic closure bounds $L_t=S_t/N$ and $U_t=(S_t+N-t)/N$, which are exact and so preserve validity. One purchased row updates every candidate, stopping one candidate does not change the random-order design for the others, and unresolved candidates at budget exhaustion are abstentions. $\square$

A stake that conditions on the proxy of the next selected row is not $\mathcal F_{t-1}$-measurable and breaks the supermartingale property; the certifier therefore uses only all-pool cheap summaries and previously purchased outcomes.

\paragraph{Proof of Theorem~\ref{thm:bestexpert}.}
A fixed convex combination of nonnegative supermartingales with unit start is again one, so $E^{\mathrm{mix}}$ is an e-process. Pathwise, $E_t^{\mathrm{mix}}\ge wE_t^{\mathrm L}$ and $E_t^{\mathrm{mix}}\ge(1-w)E_t^{\mathrm R}$, giving the displayed inequality and regret at most $c(w)$. The constant is sharp: if $E_t^{\mathrm L}/E_t^{\mathrm R}\to\infty$ then $E_t^{\mathrm{mix}}/E_t^{\mathrm L}\to w$ and the regret approaches $-\log w$, and symmetrically for the robust expert. Hence no smaller uniform constant holds for that $w$, and $c(w)$ is minimized where $-\log w=-\log(1-w)$, uniquely at $w=1/2$ with value $\log2$. $\square$

\paragraph{The advice-mismatch identity.} Proposition~\ref{prop:phase} rests on an exact one-step identity for the expected log growth of a Kelly bettor with a possibly stale forecast.

\begin{theorem}[Advice-mismatch growth identity]\label{thm:mismatch}
Consider one certification step with $Y\sim\mathrm{Ber}(p)$, $0<p<m<1$, and a ledger forecast $r\in(0,m)$. With the Kelly factor $B_r(Y)=1+\frac{m-r}{m(1-m)}(m-Y)$,
\[
    \E_p[\log B_r(Y)]=\kl(p\|m)-\kl(p\|r).
\]
The oracle forecast $r=p$ maximizes expected log growth, stale advice loses exactly $\kl(p\|r)$ relative to it, and ledger evidence grows in expectation if and only if $\kl(p\|r)<\kl(p\|m)$.
\end{theorem}

\paragraph{Proof of Theorem~\ref{thm:mismatch}.}
With $\lambda=(m-r)/\{m(1-m)\}$, $B_r(1)=1-\lambda(1-m)=r/m$ and $B_r(0)=1+\lambda m=(1-r)/(1-m)$. Hence
\begin{align*}
\E_p\log B_r(Y)
&=p\log\frac rm+(1-p)\log\frac{1-r}{1-m}\\
&=\Big[p\log\frac pm+(1-p)\log\frac{1-p}{1-m}\Big]-\Big[p\log\frac pr+(1-p)\log\frac{1-p}{1-r}\Big]
=\kl(p\|m)-\kl(p\|r).
\end{align*}
The second term is nonnegative and zero only at $r=p$. The rejection direction follows by replacing $(Y,p,m,r)$ with $(1-Y,1-p,1-m,1-r)$. $\square$

\paragraph{Proof of Corollary~\ref{cor:envelope}.}
Since $E_t^{\mathrm{mix}}\ge wE_t^{\mathrm L}$ and $E_t^{\mathrm{mix}}\ge(1-w)E_t^{\mathrm R}$ at every $t$, the mixture reaches $A$ no later than the ledger expert reaches $A/w$ or the robust expert reaches $A/(1-w)$. Taking first crossing times, with a missing crossing at time $\infty$, gives the claim. No growth-rate or overshoot assumption is used. $\square$

\paragraph{Proof of Proposition~\ref{prop:phase}.}
Write $f(r,y)=\log\{1+\lambda(r)(m-y)\}$ with $\lambda(r)=[m-r]_+/\{m(1-m)\}$, so that $f(r,1)=\log(\min\{r,m\}/m)$ and $f(r,0)=\log\{(1-\min\{r,m\})/(1-m)\}$.
(i) The ledger forecast $r=p-\eta$ is fixed, so $t^{-1}\log E_t^{\mathrm L}=t^{-1}\sum_{s\le t}f(r,Y_s)\to\E_pf(r,Y)$ by the strong law. If $r\ge m$ the stake is zero and $f\equiv0$. If $0<r<m$, Theorem~\ref{thm:mismatch} gives $\E_pf(r,Y)=\kl(p\|m)-\kl(p\|r)$. If $r\le0$ the clipped forecast is $0$, $f(0,1)=-\infty$, and the wealth is zero from the first unit loss, which occurs almost surely because $p>0$.
(ii) On $(0,p)$ the map $r\mapsto\kl(p\|r)$ is continuous and strictly decreasing from $+\infty$ to $0$, so $r_\star$ exists and is unique; on $(p,1)$ it is strictly increasing. Hence $\kl(p\|r)<\kl(p\|m)$ exactly for $r\in(r_\star,m)$, that is $\eta\in(\eta_-,\eta_+)$; $\kl(p\|r)=0$ only at $r=p$; and $\kl(p\|r)>\kl(p\|m)$ for $r\in(0,r_\star)$, that is $\eta\in(\eta_+,p)$.
(iii) The truncated robust forecast $\hat p_{t-1}=\max\{\varepsilon,\bar Y_{t-1}\}$ (with the proxy mean as the forecast at $t=1$) converges to $p$ almost surely because $p\ge\varepsilon$. On $[\varepsilon,1]$, $f(\cdot,y)$ is continuous and bounded for $y\in\{0,1\}$, because $f(r,1)\ge\log(\varepsilon/m)$, so it is uniformly continuous with some modulus $\omega$. Then $t^{-1}\sum_{s\le t}|f(\hat p_{s-1},Y_s)-f(p,Y_s)|\le t^{-1}\sum_{s\le t}\omega(|\hat p_{s-1}-p|)\to0$ by Ces\`aro averaging, and $t^{-1}\sum_{s\le t}f(p,Y_s)\to\E_pf(p,Y)=\kl(p\|m)$ by the strong law and Theorem~\ref{thm:mismatch} with $r=p$.
(iv) Pathwise $\tfrac12\max\{E_t^{\mathrm L},E_t^{\mathrm R}\}\le E_t^{\mathrm{mix}}\le\max\{E_t^{\mathrm L},E_t^{\mathrm R}\}$, so $t^{-1}\log E_t^{\mathrm{mix}}$ has the same limit as $t^{-1}\max\{\log E_t^{\mathrm L},\log E_t^{\mathrm R}\}$, namely $\max\{G_L(\eta),\kl(p\|m)\}$, which equals $\kl(p\|m)$ because $G_L(\eta)\le\kl(p\|m)$ by (ii). $\square$

\paragraph{Proof of Corollary~\ref{cor:hedgeprice}.}
For $\eta\in(\eta_-,\eta_+)$ the ledger log wealth $L_t=\sum_{s\le t}f(r,Y_s)$ is a random walk with i.i.d.\ increments of mean $G_L(\eta)>0$, whose largest increment is $f(r,0)=c_\eta$. Let $T_{\mathrm L}(A')$ be its first passage above $\log A'$; it is finite with finite mean because the drift is positive and the increments are bounded. Before passage $L_t<\log A'$, so $L_{T_{\mathrm L}\wedge n}\le\log A'+c_\eta$ for every $n$. Wald's identity \citep{wald1947} gives $G_L(\eta)\E[T_{\mathrm L}\wedge n]=\E[L_{T_{\mathrm L}\wedge n}]\le\log A'+c_\eta$, and monotone convergence gives $\E[T_{\mathrm L}(A')]\le(\log A'+c_\eta)/G_L(\eta)$. Corollary~\ref{cor:envelope} with $w=1/2$ gives $T_{\mathrm{port}}(K/\delta)\le T_{\mathrm L}(2K/\delta)$ pathwise; taking expectations with $A'=2K/\delta$ proves the first bound, and $A'=K/\delta$ gives the ledger's. Their difference is $\log2/G_L(\eta)$. At $\eta=0$, $G_L(0)=\kl(p\|m)$. Finally, $\kl(1-\beta\|\delta)=(1-\beta)\log(1/\delta)+O(1)$ as $\delta\to0$ and $\kl(p\|q_1)\to\kl(p\|m)$ as $q_1\downarrow m$, which gives the stated comparison with Theorem~\ref{thm:vigilance}. $\square$

\paragraph{Stream-level extension.} If predictable block budgets satisfy $\sum_b\delta_b\le\delta_{\mathrm{total}}$, conditional block validity and a union bound give stream-level false-certification probability at most $\delta_{\mathrm{total}}$. The reported experiments use block-level $\delta=0.05$.

\section{A matching certifier and the cost of label caps}\label{app:matching}

This appendix states and proves the upper-bound side of Theorem~\ref{thm:vigilance} in the Bernoulli certification model of Section~\ref{sec:vigilance}. For a queried label $X$ and $p_q(x)=q^x(1-q)^{1-x}$, let $z(X)=\log\{p_{q_1}(X)/p_{q_0}(X)\}$ and let $Z_{i,n}$ be the sum of the first $n$ increments at coordinate $i$. With $a=\log(M/\beta)$ and $b=\log(1/\delta)$, \VigilLR{} processes coordinates in deployment order and queries coordinate $i$ until $Z_{i,n}\le-b$, when it certifies the coordinate and continues, or $Z_{i,n}\ge a$, when it enters permanent fallback.

\begin{theorem}[Matching expected-label upper bound]\label{thm:upper}
\VigilLR{} satisfies Eq.~\ref{eq:honesty} over the composite safe and unsafe classes, and $\E_0[\sum_iN_i]\le M\{\log(1/\delta)+c_-\}/\kl(q_0\|q_1)$ with $c_-=\log\frac{1-q_0}{1-q_1}$. Consequently, for $q_0,q_1\in[\eta,1-\eta]$, $\beta\le\beta_0<1$, and $\delta\to0$, the expected-label minimax rate is $\Theta\big(M(q_1-q_0)^{-2}\log(1/\delta)\big)$.
\end{theorem}

\begin{proof}
\emph{Honesty.} Let $i^\star$ be the first unsafe coordinate reached while the procedure is active. At mean $q_1$, $\exp(-Z_{i^\star,n})$ is a nonnegative martingale with unit start, and at larger means a supermartingale by the monotone likelihood ratio of the Bernoulli family. Ville's inequality gives $\Pr(\inf_nZ_{i^\star,n}\le-b)\le\delta$. If $i^\star$ instead reaches the upper boundary, absorbing fallback prevents every later certificate, so no $\delta/M$ split is needed.
\emph{Liveness.} Under $P_0$, $\exp(Z_{i,n})$ is a nonnegative martingale, and a supermartingale at safer means, so Ville's inequality and a union bound give $P_0(\exists i,n:Z_{i,n}\ge a)\le Me^{-a}=\beta$; each walk has negative drift and reaches $-b$ almost surely unless it first reaches $a$.
\emph{Cost.} Under $P_0$ the increment mean is $-\kl(q_0\|q_1)$, and the lower crossing overshoots by at most $c_-$, so Wald's identity \citep{wald1947} on $\tau_i\wedge r$ gives $\kl(q_0\|q_1)\E_0[\tau_i\wedge r]\le b+c_-$; monotone convergence and summation give the bound. The upper boundary is reached only on the fallback event and contributes no $\log M$ term to expected safe-path cost. The rate follows by combining with Theorem~\ref{thm:vigilance}, since $\kl(q_0\|q_1)=\Theta((q_1-q_0)^2)$ on $[\eta,1-\eta]^2$ and $\kl(1-\beta\|\delta)=\log(1/\delta)(1+o(1))$.
\end{proof}

\begin{proposition}[Deterministic caps cost more]\label{prop:cap}
Let $\rho=1+(q_1-q_0)^2/\{q_0(1-q_0)\}$. An auditor that queries at most $n$ labels at each coordinate almost surely needs
\[
    n\ \ge\ \max\left\{\frac{\kl(1-\beta\|\delta)}{\kl(q_0\|q_1)},\ \frac{\kl(\delta\|1-\beta)}{\kl(q_1\|q_0)},\ \frac{\log\big(1+4M(1-\beta-\delta)^2\big)}{\log\rho}\right\},
\]
and the fixed-sample rule \VigilCap{}, which certifies a coordinate when its $n$-sample mean is at most $(q_0+q_1)/2$ and otherwise falls back permanently, satisfies Eq.~\ref{eq:honesty} whenever $n\ge2(q_1-q_0)^{-2}\max\{\log(M/\beta),\log(1/\delta)\}$. For interior parameters the cap therefore costs $\Theta\big((q_1-q_0)^{-2}\max\{\log M,\log(1/\beta),\log(1/\delta)\}\big)$ per coordinate, a $\log M$ factor that adaptive stopping removes.
\end{proposition}

\begin{proof}
The first term is Theorem~\ref{thm:vigilance} with $\E_0[N_i]\le n$. The second reverses the relative-entropy direction: $\KL(P_i^{\mathrm{tr}}\|P_0^{\mathrm{tr}})\le n\kl(q_1\|q_0)$ and data processing through the all-certified event gives $n\kl(q_1\|q_0)\ge\kl(\delta\|1-\beta)$. For the third, let $L_i=dP_i^{\mathrm{tr}}/dP_0^{\mathrm{tr}}$ and $Q=M^{-1}\sum_iP_i$. For $i\ne j$ adaptive querying preserves $\E_0[L_iL_j]=1$, and $L_i^2/\rho^{N_i}$ is a nonnegative $P_0$-martingale with $N_i\le n$, so $\E_0[L_i^2]\le\rho^n$ and $\chi^2(Q\|P_0)\le(\rho^n-1)/M$. The all-certified event has probability at least $1-\beta$ under $P_0$ and at most $\delta$ under $Q$, so $1-\beta-\delta\le\mathrm{TV}(P_0,Q)\le\frac12\sqrt{(\rho^n-1)/M}$. For the upper bound, Hoeffding's inequality \citep{hoeffding1963} and a union bound over $M$ coordinates give liveness, and the first unsafe coordinate reached passes its test with probability at most $\exp\{-n(q_1-q_0)^2/2\}\le\delta$, after which absorbing fallback prevents later certificates.
\end{proof}

\section{Additional theoretical remarks}\label{app:theoryremarks}

This section collects remarks that the main text states briefly: the scope of the lower bound, why each assumption of the validity theorem is needed, and how to read the phase diagram.

\paragraph{Scope of Theorem~\ref{thm:vigilance}.} The expectation is taken under $P_0$, the stationary law on which residual transport is exactly true. The cost appears only because the auditor must also be honest over the class containing hidden single-coordinate violations, so it applies to any use of historical labels: prediction-powered correction, active sampling, shared ledgers, and randomized validation change the adaptive query rule inside the transcript, not the information available to it. For $\Delta=q_1-q_0\to0$ the per-coordinate cost is $\Omega\big(q_0(1-q_0)\Delta^{-2}\kl(1-\beta\|\delta)\big)$, the residual-variance term one expects from a prediction-powered analysis, multiplied by $M$ coordinates. A sequential likelihood-ratio certifier with absorbing fallback attains the bound up to an overshoot constant, closing the expected-label minimax rate at $\Theta\big(M(q_1-q_0)^{-2}\log(1/\delta)\big)$, and a deterministic per-coordinate cap pays an additional $\log M$ (Appendix~\ref{app:matching}). The bound is stated for i.i.d.\ Bernoulli draws; it does not transfer verbatim to the finite without-replacement pools of Section~\ref{sec:method}, where complete enumeration resolves any candidate after $N$ labels.

\paragraph{Necessity of the assumptions of Theorem~\ref{thm:validity}.} Each assumption is needed. If rows are not sampled uniformly, the conditional mean of the next loss need not exceed $m_t(k)$ and the betting factor can have mean above one. If a stake depends on the next row, for example on its proxy value, it is no longer predictable and the supermartingale property fails. If thresholds or advice depend on target-role outcomes, the null itself becomes data-dependent and the level is no longer controlled.

\paragraph{Reading the phase diagram.} Proposition~\ref{prop:phase} says three things a reader can check against Figure~\ref{fig:phase}. History can never improve the evidence \emph{rate}: accurate advice only matches the rate the robust expert reaches by learning, so the value of good history is a finite-sample head start, which is exactly where the empirical frontiers separate (Figure~\ref{fig:frontier}). The damage from bad history is asymmetric: pessimistic stale advice makes the ledger stop betting and costs only power, while optimistic stale advice beyond $\eta_+$ makes it lose wealth, so the ledger's evidence decays to zero and it certifies only if it crosses the threshold early by chance. The portfolio's rate is the vigilance rate $\kl(p\|m)$ for every $\eta$, so it never inherits either failure. Corollary~\ref{cor:hedgeprice} makes the finite-sample price of not knowing $\eta$ explicit.

\paragraph{Stale-advice regimes in the numerical evaluation.} In the numerical evaluation of Figure~\ref{fig:phase}, exact advice gives mean labels of $430.3$ for the ledger expert, $492.3$ for the portfolio, and $790.6$ for the robust expert. At $\eta=-0.06$ the ledger never bets and does not certify within 6{,}000 labels on any path, while the portfolio needs $913.4$. At $\eta=+0.06$ the ledger fails to certify within 6{,}000 labels on $38.7\%$ of paths, while the portfolio needs $751.0$, slightly below the robust expert's $784.1$; the ledger still crosses on $61.3\%$ of paths despite its negative drift, and by Corollary~\ref{cor:envelope} the portfolio stops no later than the earlier of the two experts' crossings at level $2K/\delta$. This is the qualitative pattern of the bad-history stress test in Section~\ref{sec:hedge}.

\section{Data and comparator details}\label{app:algorithms}

This section describes how the data roles are built and how the PPRM-style comparator maps to the original method.

\paragraph{Role construction.} DICES-990 and CIFAR-10N use deterministic checksum role assignment with 40\% training, 20\% ledger, 10\% calibration, and 30\% final roles. Whole DICES conversations are assigned to one role, and every one of the 72{,}103 DICES ratings and 50{,}000 CIFAR-10N images is assigned. CIFAR-10N additionally draws proxy-training, ledger and calibration, and final outcomes from its three disjoint human annotation sets. The primary evaluation covers 18 final blocks, five seeds, nine budgets, and eight methods, and the external evaluation covers six blocks with the same seeds, budgets, and methods.

\paragraph{PPRM-style comparator.} PPRM monitors running deployed risk under shift using prediction-powered residual correction and anytime-valid CM-EB bounds \citep{zhang2026pprm}. Our comparator keeps the estimator $\widehat R_n=\overline Q_{\mathrm{pool}}+n^{-1}\sum_{t\le n}(Y_{\pi_t}-Q_{\pi_t})$ and a CM-EB-style sequential radius but targets the mean of a fixed finite pool sampled by the same uniform without-replacement prefix, with the same thresholds, seeds, candidate sharing, and budget checkpoints as the portfolio. It returns release, rejection, or abstention, so label counts are never compared at unequal decision coverage. We call it ``PPRM-style'' throughout and do not transfer the original PPRM theorem to this setting.

\section{Per-block and external results}\label{app:primaryextra}

\begin{figure}[h]
\centering
\includegraphics[width=\textwidth]{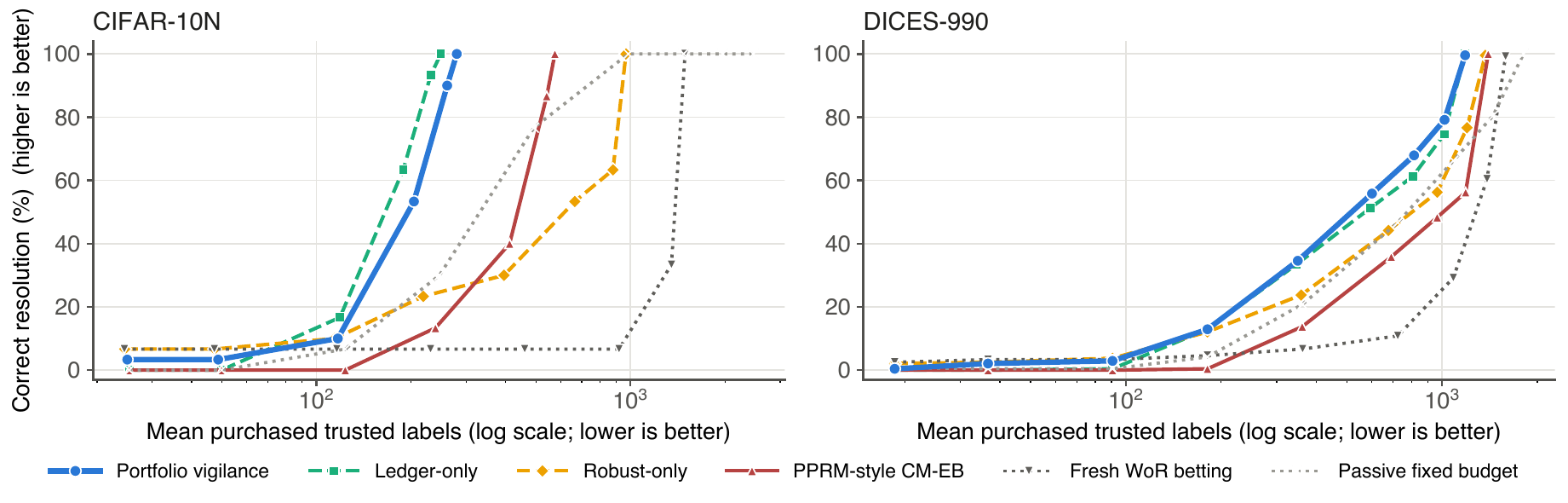}
\caption{Correct resolution (\up) against mean purchased labels (\down, log scale) across the nine budget checkpoints, averaged over blocks and seeds.}
\label{fig:frontier}
\end{figure}

Figure~\ref{fig:blocks} shows the full-budget label ratio of every held-out block. The six U.S.-locale block ratios range from $0.422$ to $0.849$, while the six India-locale ratios range from $0.858$ to $0.999$. Across the 12 blocks the ratio has Spearman correlation $-0.91$ with the distance between block risk and threshold and $+0.72$ with the ledger forecast error: the U.S. blocks are further from the decision boundary and better predicted by history, which is where betting has the most to gain. These are descriptive associations at $n=12$, not demographic inference. All six CIFAR-10N blocks favor the portfolio (ratios $0.275$ to $0.582$).

\begin{figure}[t]
\centering
\includegraphics[width=\textwidth]{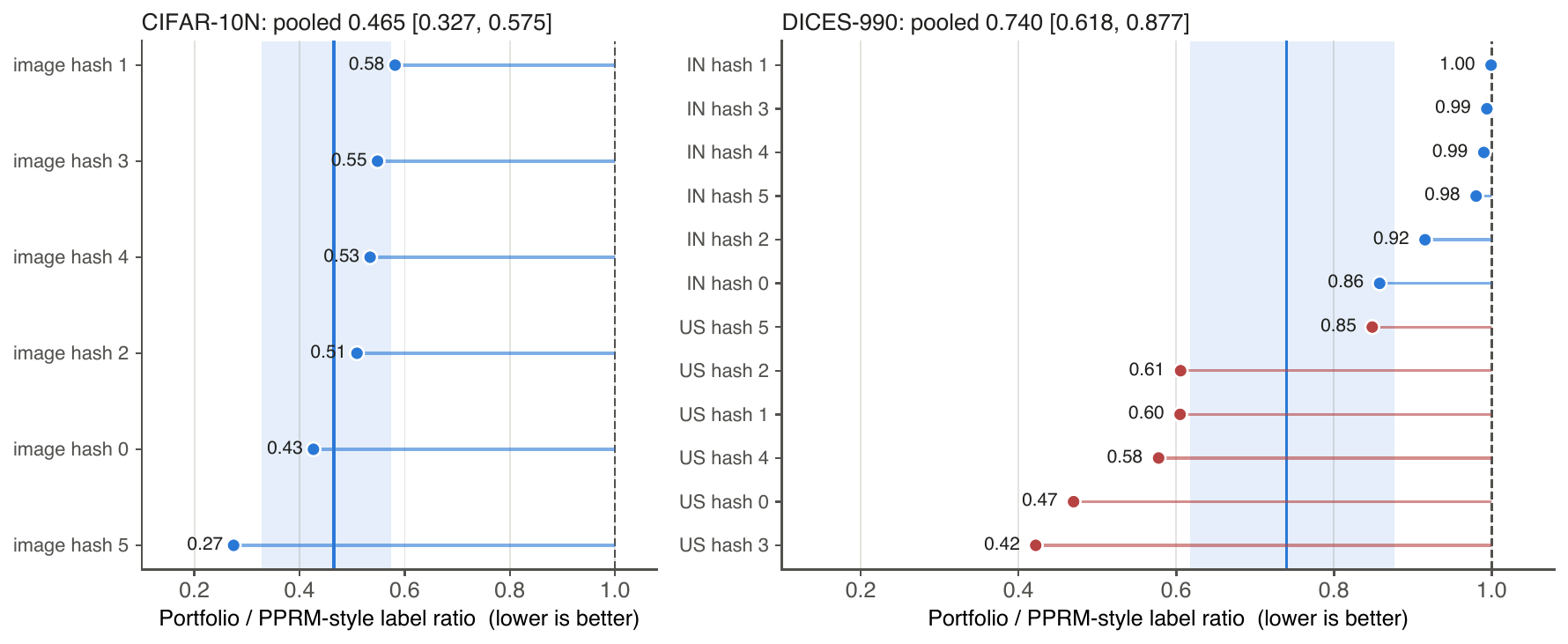}
\caption{Per-block portfolio/PPRM-style label ratios (\down) at full budget, with the pooled ratio and its 95\% interval (line and band) and parity (dashed). DICES-990: blue for India-locale, red for U.S.-locale blocks.}
\label{fig:blocks}
\end{figure}

Table~\ref{tab:external} applies the unchanged method to the six external blocks: four CIFAR-10 blocks at controlled Gaussian-noise severities $0$ to $3$ (10{,}000 rows each), SVHN as an out-of-distribution block (26{,}032 rows), and UCI Naval Propulsion (2{,}387 rows), each with a fixed candidate model and a cheap loss proxy. The gain is largest on the higher-severity CIFAR blocks, where either history or quickly learned target residuals give a decisive betting direction, and nearly vanishes on Naval, where both methods already resolve quickly. Budget checkpoints and seeds are correlated, so we do not convert the absence of false certifications into an independence-based error estimate; validity rests on Theorem~\ref{thm:validity}.

\begin{table}[t]
\caption{External blocks at full budget. Both methods resolve every decision correctly without false certification. Labels are mean $\pm$ SD over five seeds; the ratio is of mean labels.}
\label{tab:external}
\small
\begin{center}
\begin{tabular}{lrrrc}
\toprule
Block & Portfolio labels \down & PPRM-style labels \down & Label ratio \down & Decision (desc.) \\
\midrule
CIFAR-10, noise severity 0 & $2{,}925.2\pm2{,}581.2$ & $4{,}684.6\pm2{,}288.4$ & $0.624$ & release \\
CIFAR-10, noise severity 1 & $165.8\pm66.6$ & $354.4\pm71.4$ & $0.468$ & release \\
CIFAR-10, noise severity 2 & $81.6\pm46.6$ & $221.2\pm15.1$ & $0.369$ & release \\
CIFAR-10, noise severity 3 & $60.0\pm0.0$ & $202.6\pm0.5$ & $\mathbf{0.296}$ & release \\
SVHN (out of distribution) & $5{,}752.4\pm2{,}671.8$ & $8{,}981.6\pm2{,}821.8$ & $0.641$ & reject \\
Naval regression & $305.4\pm59.7$ & $333.6\pm13.2$ & $0.916$ & release \\
\bottomrule
\end{tabular}
\end{center}
\end{table}

\section{Corrupted advice, mixture weight, and scaling}\label{app:stress}

These diagnostics evaluate portfolio vigilance on synthetic finite pools, separately from the natural-data analysis. Table~\ref{tab:advice} reports the advice-corruption grid and Table~\ref{tab:scaling} the gap and error-budget scaling. Figure~\ref{fig:hedge} plots the corruption grid.

\begin{figure}[t]
\centering
\includegraphics[width=0.62\textwidth]{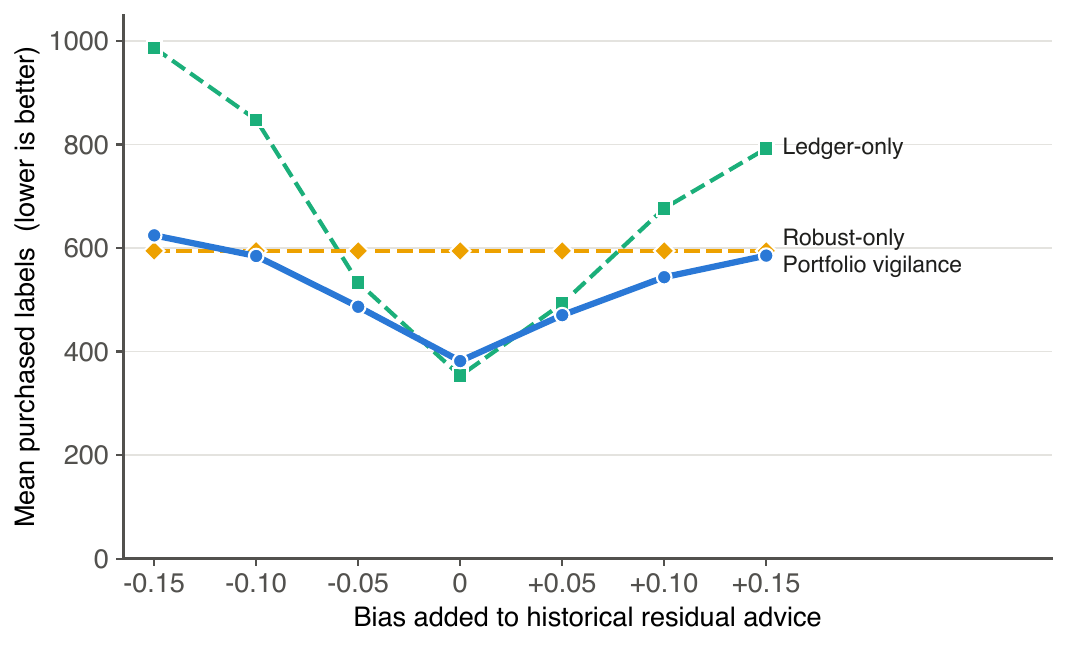}
\caption{Mean purchased labels (\down) as the historical advice is biased, on synthetic finite pools.}
\label{fig:hedge}
\end{figure}

\begin{table}[h]
\caption{Mean purchased labels (\down) under biased advice ($N=1{,}200$, threshold $0.20$). Last column: portfolio labels over the cheaper component.}
\label{tab:advice}
\small
\begin{center}
\begin{tabular}{rrrrr}
\toprule
Advice bias & Ledger-only \down & Portfolio \down & Robust-only \down & Portfolio / better component \\
\midrule
$-0.15$ & $986.7$ & $624.5$ & $594.5$ & $1.050$ \\
$-0.10$ & $847.7$ & $584.6$ & $594.5$ & $0.983$ \\
$-0.05$ & $533.5$ & $486.6$ & $594.5$ & $0.912$ \\
$0$ & $353.4$ & $381.7$ & $594.5$ & $1.080$ \\
$+0.05$ & $492.8$ & $470.4$ & $594.5$ & $0.955$ \\
$+0.10$ & $676.7$ & $543.7$ & $594.5$ & $0.915$ \\
$+0.15$ & $792.3$ & $585.4$ & $594.5$ & $0.985$ \\
\bottomrule
\end{tabular}
\end{center}
\end{table}

\paragraph{Mixture weight.} On the same construction with 500 replications, a safe pool (risk $0.20$) with exact advice needs $464.5$, $113.0$, and $101.8$ mean labels at ledger weight $0$, $0.5$, and $1$; with misleading advice (forecast $0.5$) it needs $464.5$, $493.1$, and $734.5$. An unsafe pool (risk $0.40$) needs $270.5$, $110.6$, and $98.2$ with exact advice and $270.5$, $308.5$, and $695.5$ with misleading advice (forecast $0.1$). Empirical error rates are at most $0.016$ in every cell. The weight trades efficiency between the two advice regimes and does not affect validity.

\paragraph{Risk gap.} With safe risk $0.20$, unsafe risk $0.40$, and thresholds giving gaps from $0.02$ to $0.15$, mean labels with aligned advice fall from $679.4$ to $61.0$ on the safe side and from $672.8$ to $48.4$ on the unsafe side; misleading advice raises every cost but preserves the monotone pattern, so distance from the decision boundary remains the dominant difficulty.

\begin{table}[h]
\caption{Portfolio scaling with aligned advice ($N=5{,}000$, $\beta=0.10$, threshold $0.30$) against the lower bound of Theorem~\ref{thm:vigilance}.}
\label{tab:scaling}
\small
\begin{center}
\begin{tabular}{rrrrr}
\toprule
Gap & $\delta$ & Mean labels \down & Lower bound & Ratio \\
\midrule
$0.03$ & $0.05$ & $1{,}145.9$ & $281.9$ & $4.06$ \\
$0.05$ & $0.05$ & $468.0$ & $102.4$ & $4.57$ \\
$0.075$ & $0.05$ & $220.1$ & $45.9$ & $4.80$ \\
$0.10$ & $0.05$ & $121.9$ & $26.0$ & $4.69$ \\
$0.15$ & $0.05$ & $54.1$ & $11.6$ & $4.67$ \\
\midrule
$0.10$ & $0.01$ & $188.8$ & $41.7$ & $4.52$ \\
$0.10$ & $0.10$ & $91.0$ & $19.2$ & $4.74$ \\
\bottomrule
\end{tabular}
\end{center}
\end{table}

\section{Natural-data sensitivity}\label{app:naturalsensitivity}

These analyses keep the roles, outcomes, sampling design, random orders, and certifier of the primary study and vary one factor at a time. Table~\ref{tab:margins} varies the threshold margin added to ledger risk.

\begin{table}[h]
\caption{Threshold-margin sensitivity. Res@20: correct resolution at 20\% budget (portfolio/PPRM). The portfolio makes no false certification at any margin.}
\label{tab:margins}
\small
\begin{center}
\begin{tabular}{lrrrrr}
\toprule
Dataset & Margin & Ratio \down & Portfolio labels \down & PPRM labels \down & Res@20 \up \\
\midrule
CIFAR-10N & $0.020$ & $0.428$ & $938.9$ & $2{,}032.7$ & $0.100/0.000$ \\
CIFAR-10N & $0.030$ & $0.403$ & $624.4$ & $1{,}452.8$ & $0.367/0.000$ \\
CIFAR-10N & $0.050$ & $0.465$ & $279.9$ & $574.6$ & $0.900/0.400$ \\
CIFAR-10N & $0.075$ & $0.451$ & $158.3$ & $314.3$ & $1.000/0.900$ \\
CIFAR-10N & $0.100$ & $0.414$ & $98.4$ & $211.6$ & $1.000/1.000$ \\
DICES-990 & $0.020$ & $0.729$ & $1{,}367.1$ & $1{,}782.2$ & $0.192/0.046$ \\
DICES-990 & $0.030$ & $0.674$ & $1{,}245.8$ & $1{,}707.9$ & $0.225/0.054$ \\
DICES-990 & $0.050$ & $0.740$ & $1{,}182.2$ & $1{,}398.3$ & $0.346/0.138$ \\
DICES-990 & $0.075$ & $0.709$ & $1{,}048.7$ & $1{,}241.9$ & $0.513/0.275$ \\
DICES-990 & $0.100$ & $0.649$ & $860.5$ & $1{,}089.5$ & $0.642/0.458$ \\
\bottomrule
\end{tabular}
\end{center}
\end{table}

\paragraph{CIFAR-10N observability.} The primary proxy uses canonical class plus released worker and annotation-time side information. Retrained on canonical class only, the portfolio/PPRM ratio is $0.4658$ ($279.8$ against $574.3$ labels), essentially the primary $0.4654$, with $90\%$ against $40\%$ resolution at 20\% budget and zero portfolio false certifications.

\paragraph{A non-LLM evaluator-version change on DICES.} We define an older cheap evaluator as a 10k-feature unigram TF-IDF logistic model and the current evaluator as the primary 30k-feature unigram-plus-bigram model, compute historical advice against the old evaluator, and expose the current evaluator on the target pool. At full budget the portfolio needs $1{,}184.8$ labels with old-to-current advice against $1{,}182.2$ with aligned advice, ledger-only $1{,}181.1$ against $1{,}181.4$, and robust-only $1{,}368.8$ in both, with no false certification and the same single false rejection. This modest change leaves the historical forecast informative, so it shows stability under one natural evaluator replacement rather than worst-case drift.

\paragraph{DICES block heterogeneity.} Per-block ratios are $0.858$, $0.999$, $0.916$, $0.994$, $0.990$, and $0.980$ for India-locale hash blocks $0$ to $5$, and $0.470$, $0.605$, $0.606$, $0.422$, $0.578$, and $0.849$ for the U.S.-locale blocks. The ratio has Spearman correlation $-0.91$ with both the mean and the minimum absolute distance between block risk and threshold, and $+0.72$ with the mean absolute ledger forecast error.

\section{Controlled LLM evaluator-configuration shift}\label{app:llmshift}

We design both LLM studies after fixing the primary protocol, so they complement rather than replace the DICES-990 and CIFAR-10N results. An initial single-run study judges the DICES ledger and final items once per rubric; the repeated-judge study reported in Section~\ref{sec:sensitivity} judges every item three times per rubric and adds ToxicChat.

\subsection{Initial single-run study}

\paragraph{Design and blinding.} Both judges use the same model, GPT-5.6 Luna, with the same reasoning setting and a strict structured-output schema. Judge A uses a legacy narrow safety rubric and Judge B an updated contextual-risk rubric, and both return probabilities and NO/YES/UNSURE labels for harmful content, unfair bias, misinformation, and overall unsafe content. The judge-visible payload contains only an opaque record identifier, the conversation context, and the assistant response; human labels, role assignment, thresholds, historical outputs, certification state, and the research hypothesis are withheld. The deterministic DICES item split supplies 201 ledger and 298 final conversations, 499 per judge and 998 judge calls in total, joined back to 14{,}620 ledger and 21{,}702 final rating rows. The analysis reuses the primary random orders, budgets, $\delta$, $\beta$, and ledger-risk-plus-$0.05$ thresholds. Judge A supplies the historical ledger and Judge B the current proxy in $A\!\to\!B$; $A\!\to\!A$ and $B\!\to\!B$ are stationary controls.

\paragraph{Provenance.} All 998 judge calls complete successfully and report the intended model, and re-parsing every raw response reproduces the normalized judgments exactly. The two judges consume 659{,}890 input and 76{,}998 output tokens in total.

\paragraph{Shift diagnostics.} On the 298 final conversations the mean absolute A/B probability difference is $0.0708$ and the discrete A/B disagreement rate across the four outcomes is $8.98\%$. Residual mismatch for harmful content, unfair bias, misinformation, and overall unsafe content is $(0.0307,0.0118,0.0083,0.0103)$ under $A\!\to\!A$, $(0.0162,0.0474,0.0579,0.0873)$ under $A\!\to\!B$, and $(0.0243,0.0006,0.0142,0.0003)$ under $B\!\to\!B$. For these diagnostics larger values indicate more evaluator change, not better performance. The update is not larger on every outcome, but it creates clear mismatch for unfair bias, misinformation, and especially overall safety.

\begin{table}[h]
\caption{LLM evaluator-configuration study at full budget. FC/FR over 240 decisions; P/P: portfolio/PPRM-style label ratio with its 95\% interval.}
\label{tab:llmshift}
\small
\begin{center}
\begin{tabular}{llrrrrrl}
\toprule
Condition & Method & Labels \down & Resolution \up & FC \down & FR \down & AUC \up & P/P ratio \down \\
\midrule
$A\!\to\!A$ & Portfolio & $1{,}214.9$ & $0.9958$ & $0$ & $1$ & $0.559$ & $0.674$ $[0.519,0.846]$ \\
& Ledger-only & $1{,}323.4$ & $1.0000$ & $0$ & $0$ & $0.498$ & \\
& Robust-only & $1{,}394.1$ & $0.9958$ & $0$ & $1$ & $0.505$ & \\
& PPRM-style CM-EB & $1{,}545.1$ & $1.0000$ & $0$ & $0$ & $0.336$ & \\
& Fresh WoR betting & $1{,}586.1$ & $0.9917$ & $1$ & $1$ & $0.316$ & \\
& Passive fixed budget & $1{,}808.5$ & $1.0000$ & $0$ & $0$ & $0.521$ & \\
\midrule
$A\!\to\!B$ & Portfolio & $1{,}304.5$ & $0.9958$ & $0$ & $1$ & $0.545$ & $0.779$ $[0.629,0.924]$ \\
& Ledger-only & $1{,}458.3$ & $1.0000$ & $0$ & $0$ & $0.452$ & \\
& Robust-only & $1{,}383.5$ & $0.9958$ & $0$ & $1$ & $0.491$ & \\
& PPRM-style CM-EB & $1{,}551.8$ & $1.0000$ & $0$ & $0$ & $0.325$ & \\
& Fresh WoR betting & $1{,}586.1$ & $0.9917$ & $1$ & $1$ & $0.316$ & \\
& Passive fixed budget & $1{,}808.5$ & $1.0000$ & $0$ & $0$ & $0.521$ & \\
\midrule
$B\!\to\!B$ & Portfolio & $1{,}203.5$ & $0.9958$ & $0$ & $1$ & $0.573$ & $0.663$ $[0.510,0.838]$ \\
& Ledger-only & $1{,}268.8$ & $1.0000$ & $0$ & $0$ & $0.522$ & \\
& Robust-only & $1{,}383.5$ & $0.9958$ & $0$ & $1$ & $0.491$ & \\
& PPRM-style CM-EB & $1{,}551.8$ & $1.0000$ & $0$ & $0$ & $0.325$ & \\
& Fresh WoR betting & $1{,}586.1$ & $0.9917$ & $1$ & $1$ & $0.316$ & \\
& Passive fixed budget & $1{,}808.5$ & $1.0000$ & $0$ & $0$ & $0.521$ & \\
\bottomrule
\end{tabular}
\end{center}
\end{table}

\begin{figure}[h]
\centering
\includegraphics[width=\textwidth]{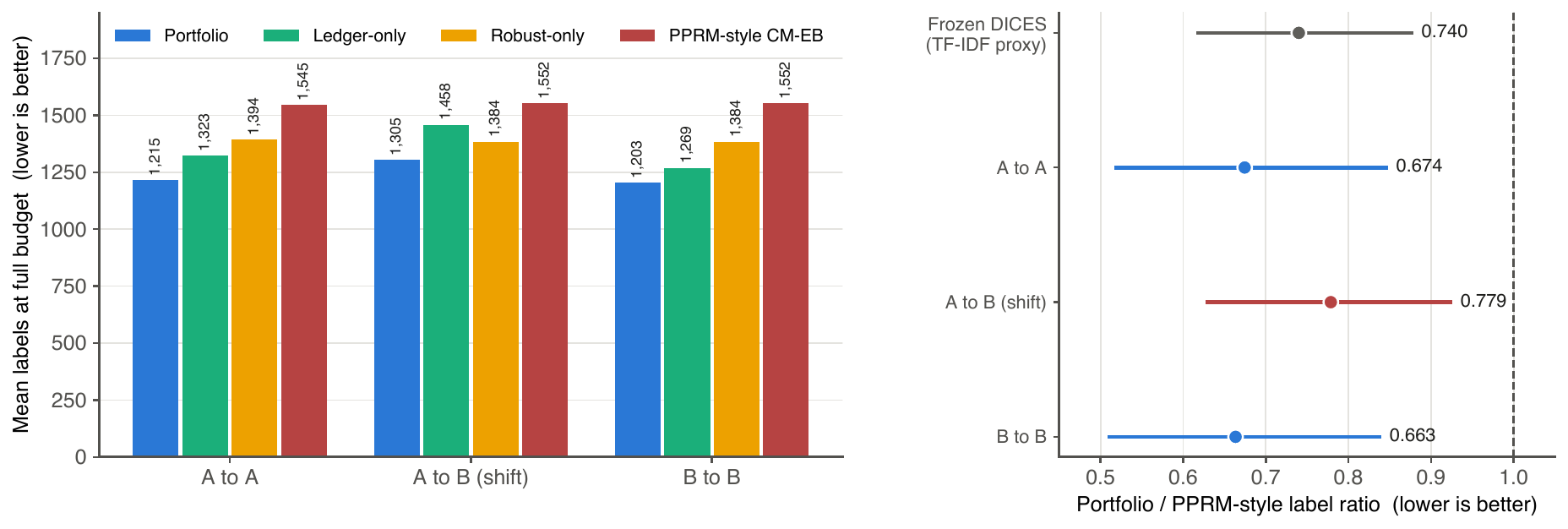}
\caption{LLM evaluator-configuration study. Left: mean full-budget labels (\down). Right: portfolio/PPRM-style label ratio (\down) with 95\% intervals, beside the primary DICES-990 result.}
\label{fig:llmshift}
\end{figure}

\paragraph{Paired contrasts.} Table~\ref{tab:llmshift} gives full-budget results for every method, and Figure~\ref{fig:llmshift} compares the portfolio's label ratio across conditions with the primary result. Under $A\!\to\!B$, paired block-bootstrap label differences (portfolio minus comparator) are $-153.8$ against ledger-only (95\% CI $[-288.8,-40.1]$), $-79.0$ against robust-only ($[-148.3,-19.2]$), and $-247.2$ against PPRM ($[-426.8,-101.1]$), and portfolio AUC exceeds PPRM by $0.2196$ ($[0.1720,0.2729]$). Relative to $A\!\to\!A$, the ledger-only expert needs $134.9$ more labels under the shift ($[2.8,291.1]$) and the portfolio $89.6$ more ($[-2.7,206.6]$); relative to $B\!\to\!B$ the increases are $189.5$ ($[33.5,366.9]$) and $101.1$ ($[23.3,204.1]$). Robust-only and PPRM are identical between $A\!\to\!B$ and $B\!\to\!B$ by construction, since neither uses A-residual advice.

\paragraph{Reading.} The pattern is the one \Mtwo{} predicts: stale advice costs the ledger expert efficiency without touching validity, and the mixture loses less because it retains the target-adaptive expert. The effect is modest. The portfolio's advantage over PPRM under the shift ($0.779$) is smaller than under either stationary control and than in the primary study ($0.740$), and every condition's interval overlaps the others. The portfolio's one false rejection occurs in the same U.S.-locale block under the same random order in all three conditions, so it does not arise from the rubric change. Fresh WoR betting's one false certification, in an India-locale block, likewise recurs in every condition.

\paragraph{Integrity.} The study covers the complete design of three conditions, 12 blocks, five seeds, nine budgets, and six methods, and an independent check recomputes every reported summary from the per-block results.

\subsection{Repeated-judge study}

\paragraph{Design.} The model, rubrics, schema, and blinding are those of the single-run study. Each rubric judges every item in three new independent runs (A2 to A4 and B2 to B4): all 990 DICES conversations and all 5{,}654 human-annotated ToxicChat-0124 examples. The single-run outputs are kept separate and are not reused as repetitions. ToxicChat is split by stable hash, without inspecting labels, into 2{,}245 proxy-training, 1{,}169 ledger, 547 calibration, and 1{,}693 final examples, and the final role forms eight hash blocks; DICES keeps its 12 final blocks. The monitors use the mean of the three runs of each rubric, the five primary random orders, and the nine budgets, and they add a without-replacement Hoeffding CS to the comparators. The written protocol for this study planned 50 random orders; we ran the five primary orders. Orders are averaged within each block before $10{,}000$ block-bootstrap replicates, so the reduced number of orders affects Monte Carlo stability but not the inferential unit. An automated audit confirms the complete design of 51{,}840 DICES and 34{,}560 ToxicChat method rows and the seed list.

\paragraph{Rubric change against run-to-run variation.} Table~\ref{tab:judgerepeat} compares judge scores between runs of the same rubric with scores between rubrics, using $10{,}000$ item-bootstrap replicates. The between-rubric change exceeds the within-rubric change in every role of both datasets; the table shows the final roles.

\begin{table}[h]
\caption{Judge run-to-run variation against rubric change on final items: mean absolute score change and binary-label disagreement within rubric A, within rubric B, and between rubrics, and the excess of the between-rubric value over the within-rubric value with its 95\% interval.}
\label{tab:judgerepeat}
\small
\begin{center}
\begin{tabular}{llrrrl}
\toprule
Dataset & Metric & Within A & Within B & Between & Excess \\
\midrule
DICES-990 ($n=298$) & Score change & $0.0181$ & $0.0365$ & $0.0749$ & $0.0476$ $[0.0331,0.0633]$ \\
& Label disagreement & $0.0162$ & $0.0336$ & $0.0910$ & $0.0661$ $[0.0474,0.0864]$ \\
ToxicChat ($n=1{,}693$) & Score change & $0.0139$ & $0.0136$ & $0.0516$ & $0.0379$ $[0.0301,0.0460]$ \\
& Label disagreement & $0.0110$ & $0.0122$ & $0.0566$ & $0.0450$ $[0.0357,0.0546]$ \\
\bottomrule
\end{tabular}
\end{center}
\end{table}

\paragraph{Full results.} Table~\ref{tab:repeatedfull} gives full-budget label costs in all three conditions. The pessimistic cost charges the full block for any false certification, false rejection, or unresolved decision. The portfolio makes no false certification in any condition; on DICES-990 it makes one false rejection in each condition, as does robust-only. Moving from $B\!\to\!B$ to $A\!\to\!B$ raises the ledger-only cost from $1{,}274.7$ to $1{,}445.9$ labels on DICES-990 and from $130.6$ to $150.2$ on ToxicChat, while the portfolio rises from $1{,}214.3$ to $1{,}305.2$ and from $130.7$ to $137.6$.

\begin{table}[h]
\caption{Repeated-judge study at full budget: mean labels (\down) and pessimistic labels (\down) over blocks and the five random orders. FC/FR: false certifications/rejections over all runs. Robust-only, PPRM, and the fresh and passive methods do not use A-residual advice, so their $A\!\to\!B$ and $B\!\to\!B$ values coincide.}
\label{tab:repeatedfull}
\small
\setlength{\tabcolsep}{3.5pt}
\begin{center}
\begin{tabular}{llrrrrrrr}
\toprule
& & \multicolumn{3}{c}{DICES-990} & & \multicolumn{3}{c}{ToxicChat} \\
\cmidrule{3-5}\cmidrule{7-9}
Condition & Method & Labels & Pessim. & FC/FR & & Labels & Pessim. & FC/FR \\
\midrule
$A\!\to\!A$ & Portfolio & $1{,}201.1$ & $1{,}220.8$ & $0/1$ & & $131.5$ & $131.5$ & $0/0$ \\
& Ledger-only & $1{,}282.8$ & $1{,}282.8$ & $0/0$ & & $129.0$ & $129.0$ & $0/0$ \\
& Robust-only & $1{,}369.2$ & $1{,}389.0$ & $0/1$ & & $129.5$ & $134.3$ & $0/1$ \\
& PPRM-style CM-EB & $1{,}536.9$ & $1{,}536.9$ & $0/0$ & & $210.7$ & $210.7$ & $0/0$ \\
\midrule
$A\!\to\!B$ & Portfolio & $1{,}305.2$ & $1{,}325.1$ & $0/1$ & & $137.6$ & $137.6$ & $0/0$ \\
& Ledger-only & $1{,}445.9$ & $1{,}445.9$ & $0/0$ & & $150.2$ & $150.2$ & $0/0$ \\
& Robust-only & $1{,}414.2$ & $1{,}434.4$ & $0/1$ & & $137.5$ & $137.5$ & $0/0$ \\
& PPRM-style CM-EB & $1{,}544.3$ & $1{,}544.3$ & $0/0$ & & $208.5$ & $208.5$ & $0/0$ \\
\midrule
$B\!\to\!B$ & Portfolio & $1{,}214.3$ & $1{,}234.3$ & $0/1$ & & $130.7$ & $130.7$ & $0/0$ \\
& Ledger-only & $1{,}274.7$ & $1{,}274.7$ & $0/0$ & & $130.6$ & $130.6$ & $0/0$ \\
& Robust-only & $1{,}414.2$ & $1{,}434.4$ & $0/1$ & & $137.5$ & $137.5$ & $0/0$ \\
& PPRM-style CM-EB & $1{,}544.3$ & $1{,}544.3$ & $0/0$ & & $208.5$ & $208.5$ & $0/0$ \\
\midrule
All & Fresh WoR betting & $1{,}586.1$ & $1{,}595.7$ & $1/1$ & & $173.3$ & $178.1$ & $0/1$ \\
& WoR Hoeffding CS & $1{,}808.5$ & $1{,}808.5$ & $0/0$ & & $211.6$ & $211.6$ & $0/0$ \\
& Passive fixed budget & $1{,}808.5$ & $1{,}808.5$ & $0/0$ & & $211.6$ & $211.6$ & $0/0$ \\
\bottomrule
\end{tabular}
\end{center}
\end{table}

\paragraph{Paired comparisons under $A\!\to\!B$.} Table~\ref{tab:repeatedpaired} gives portfolio-minus-comparator label differences with block-bootstrap intervals, the number of the nine pairings of one A run with one B run in which the portfolio is cheaper, and the range of the difference when any one of the five random orders is left out. Every comparison favors the portfolio except robust-only on ToxicChat, where the interval, the pairing count, and the leave-one-order-out range all include no difference. Pessimistic costs give the same conclusions: on DICES-990 the differences are $-120.8$ ($[-218.0,-40.8]$) against ledger-only, $-109.3$ ($[-199.4,-32.9]$) against robust-only, and $-219.2$ ($[-349.7,-111.8]$) against PPRM, and on ToxicChat they equal the ordinary differences.

\begin{table}[h]
\caption{Portfolio minus comparator, full-budget labels under $A\!\to\!B$ (negative favors the portfolio). CI: block-bootstrap 95\% interval after averaging the five random orders. Pairings: A/B run pairings (of nine) in which the portfolio is cheaper. LOO: range over leaving out one random order.}
\label{tab:repeatedpaired}
\small
\begin{center}
\begin{tabular}{llrlrl}
\toprule
Dataset & Comparator & Difference & 95\% CI & Pairings & LOO range \\
\midrule
DICES-990 & Ledger-only & $-140.7$ & $[-256.2,-46.2]$ & $9/9$ & $[-148.8,-132.1]$ \\
& Robust-only & $-109.1$ & $[-199.3,-32.6]$ & $9/9$ & $[-118.6,-97.2]$ \\
& PPRM-style CM-EB & $-239.1$ & $[-394.4,-114.2]$ & $9/9$ & $[-272.4,-219.9]$ \\
& WoR Hoeffding CS & $-503.3$ & & $9/9$ & $[-520.4,-488.7]$ \\
\midrule
ToxicChat & Ledger-only & $-12.7$ & $[-20.8,-5.9]$ & $9/9$ & $[-14.9,-10.3]$ \\
& Robust-only & $+0.05$ & $[-4.65,4.05]$ & $4/9$ & $[-3.22,2.75]$ \\
& PPRM-style CM-EB & $-70.9$ & $[-97.1,-42.6]$ & $9/9$ & $[-74.2,-68.3]$ \\
& WoR Hoeffding CS & $-74.1$ & & $9/9$ & $[-77.5,-69.9]$ \\
\bottomrule
\end{tabular}
\end{center}
\end{table}

\paragraph{The price of trusted transport.}\label{app:rho} The same data illustrate \Mone. Theorem~\ref{thm:trustedtransport} says the zero-label certifiable set shrinks as the declared radius $\rho$ grows; Figure~\ref{fig:rho} sets this against the fraction of cells for which the bridge happens to hold on the realized target data, a fact that only the target labels transport was meant to save can reveal. At $\rho=0$ a zero-label certificate would cover $81\%$ and $77\%$ of the DICES-990 candidate-block cells under rubrics A and B, although the bridge holds for only $35\%$ and $50\%$; every certificate issued where the bridge fails is unjustified. On DICES-990 the radius at which the bridge first holds everywhere ($0.12$ and $0.10$) leaves $0\%$ and $2\%$ of the cells certifiable, so a bridge wide enough to be true removes the zero-label advantage. ToxicChat, with a single outcome and eight blocks, is more forgiving: a radius of $0.05$ (rubric A) or $0.02$ (rubric B) holds on every block and still certifies $62.5\%$ and $100\%$ of them, but nothing in the cheap data identifies that radius in advance (Proposition~\ref{prop:nolabelfree}).

\begin{figure}[t]
\centering
\includegraphics[width=\textwidth]{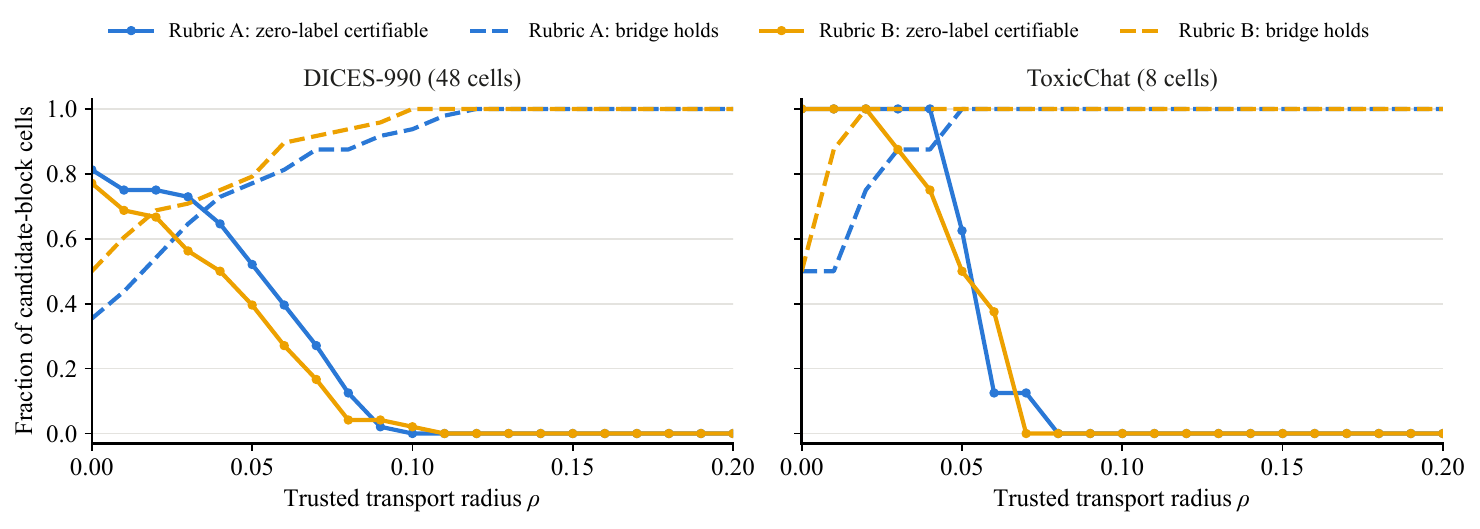}
\caption{Trusted-transport radius $\rho$ against the fraction of candidate-block cells that a zero-label certificate would certify (solid) and the fraction for which the bridge $\mu_E^{\mathrm{tar}}\le\mu_E^{\mathrm{hist}}+\rho$ holds on the realized target data (dashed), under each rubric's stationary judge grid of the repeated-judge study. The certificate uses the historical mean residual in place of the upper bound $U_H$, which makes it optimistic. A radius wide enough to hold removes most or all zero-label certificates on DICES-990.}
\label{fig:rho}
\end{figure}

\paragraph{Endpoint calibration comparison.} Noisy-but-Valid \citep{feng2026noisy} tests a fixed endpoint rather than monitoring sequentially, so we report it only as a diagnostic. With its calibration-set estimates of the judge's true- and false-positive rates, it certifies none of the ten rubric-by-endpoint cases on ToxicChat and on the DICES conversation-level majority endpoints, including the eight whose actual risk is below the threshold, so it makes no false certification but also no correct certification.

\section{Hidden failure of unverified transport}\label{app:hiddenfailure}

This study uses an earlier certifier family and different data from the rest of the paper, so we do not pool its results with Section~\ref{sec:results}; it asks only whether trusting a calibrated bridge without verifying it fails on real data. The data are ACSIncome and ACSPublicCoverage \citep{ding2021folktables} for four states with year-separated roles (2014 training and ledger, 2015 bridge calibration, 2016 to 2018 held-out tests), logistic candidates, and a 60-cell descriptor partition that defines the bridge. \TransportFirst{} releases every candidate the calibrated bridge covers and never re-checks a released candidate with labels; \VigilCap{} is a fixed-sample variant that also releases bridge-covered candidates; the exploratory \VigilCell{} verifies each bridge-covered cell with a sequential test on fresh labels. The shared-ledger oracle uses the bridge with zero labels, and the proxy-only control trusts the cheap scores alone.

For every held-out block, an injected copy raises the trusted losses of one proxy-score decile cell until every candidate exceeds its threshold by $0.02$, leaving proxies, descriptors, and the bridge untouched: the real-data analogue of the hidden alternatives in Section~\ref{sec:vigilance}. At the prespecified release margin of $0.02$ the calibrated bridge (width $0.048$) is wider than the margin, so it transports nothing and no method false-certifies. The larger margins, the calibrated $0.038$ and $0.05$, are exploratory and were chosen after this diagnostic.

\begin{table}[h]
\caption{Hidden-failure stress test on Folktables: false certification on injected blocks (120 per margin) and fresh labels on paired clean blocks. Margins other than $0.02$ are exploratory.}
\label{tab:injection}
\small
\begin{center}
\begin{tabular}{lrrrrrr}
\toprule
& \multicolumn{2}{c}{Margin $0.02$} & \multicolumn{2}{c}{Margin $0.038$} & \multicolumn{2}{c}{Margin $0.05$} \\
\cmidrule(lr){2-3}\cmidrule(lr){4-5}\cmidrule(lr){6-7}
Method & FC \down & Labels \down & FC \down & Labels \down & FC \down & Labels \down \\
\midrule
\TransportFirst{} & $0.000$ & $48{,}275$ & $0.417$ & $29{,}525$ & $\mathbf{0.750}$ & $10{,}816$ \\
\VigilCap{} & $0.000$ & $20{,}472$ & $0.417$ & $12{,}795$ & $\mathbf{0.750}$ & $5{,}118$ \\
\VigilCell{} (exploratory) & $0.000$ & $33{,}552$ & $0.000$ & $22{,}838$ & $0.000$ & $11{,}933$ \\
Selective residual CS & $0.000$ & $7{,}924$ & $0.000$ & $6{,}195$ & $0.000$ & $4{,}733$ \\
PPRM-reset & $0.000$ & $6{,}813$ & $0.000$ & $4{,}645$ & $0.000$ & $3{,}362$ \\
Betting monitor & $0.000$ & $18{,}749$ & $0.000$ & $8{,}695$ & $0.000$ & $10{,}644$ \\
Fresh-label Hoeffding & $0.000$ & $35{,}768$ & $0.000$ & $23{,}679$ & $0.000$ & $16{,}140$ \\
\midrule
Shared-ledger oracle & $0.000$ & $0$ & $0.417$ & $0$ & $\mathbf{0.750}$ & $0$ \\
Proxy-only control & $1.000$ & $0$ & $1.000$ & $0$ & $1.000$ & $0$ \\
\bottomrule
\end{tabular}
\end{center}
\end{table}

\begin{figure}[h]
\centering
\includegraphics[width=0.62\textwidth]{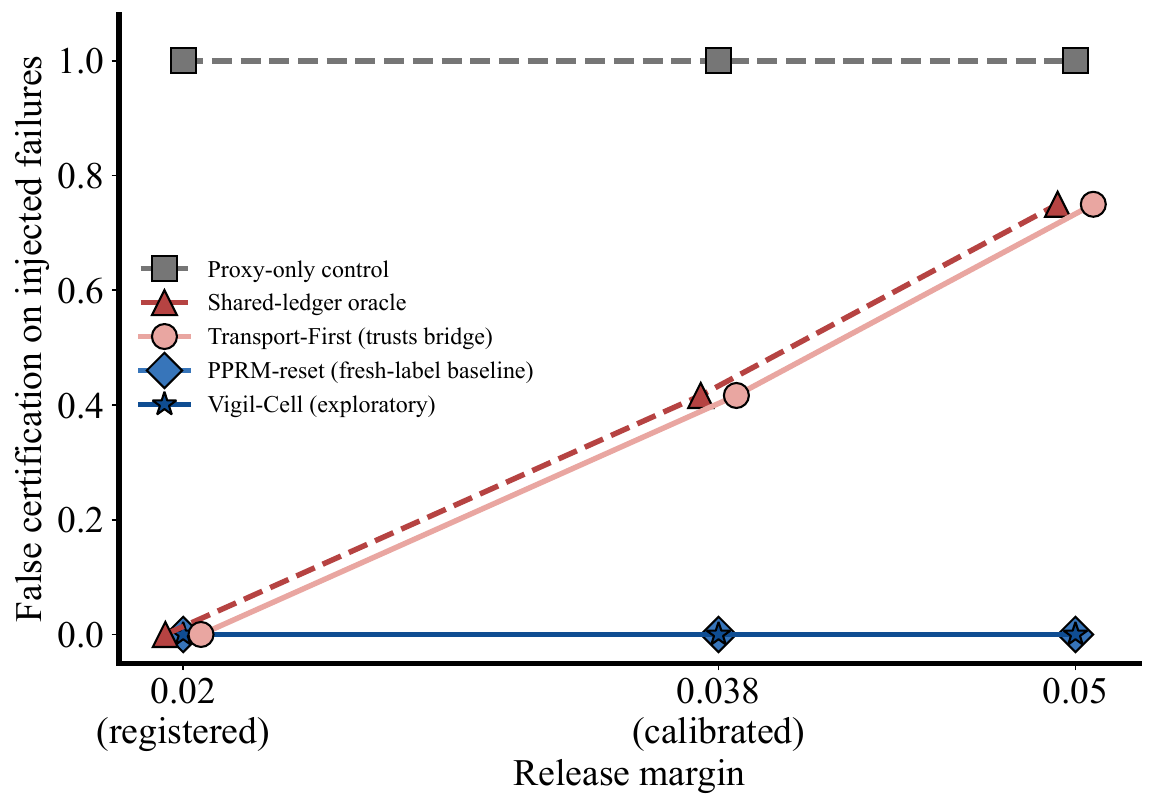}
\caption{False certification on injected failures across the three tested margins (Table~\ref{tab:injection}). The shared-ledger oracle and \TransportFirst{}, which trusts the bridge without verifying it, rise together once the bridge transports; fresh-label baselines and the exploratory cell-verifying certifier stay at zero.}
\label{fig:margin-collapse}
\end{figure}

Wherever the bridge transports, \TransportFirst{} is block for block identical to the no-label oracle, false-certifying $0.417$ and $0.750$ of the injected blocks, while every fresh-label baseline catches all of them (Figure~\ref{fig:margin-collapse}). The exploratory \VigilCell{} stays at zero false certification on all 120 injected blocks at every margin; at margin $0.05$, where the bridge covers 90 of 120 blocks, it uses $1{,}489\pm96$ labels on covered blocks against $1{,}685\pm126$ for PPRM and rejects an injected failure with a median of $346$ labels against $7{,}175$. Aggregated over all blocks PPRM remains cheaper ($3{,}362$ against $11{,}933$), and because margin $0.05$ was chosen after the diagnostics, this comparison is post hoc. These failures of blind transport motivate treating untrusted history as advice rather than as part of validity.

\end{document}

%% file: math_commands.tex
\usepackage{amsmath,amsfonts,bm}

\def\eqref#1{equation~\ref{#1}}

\def\1{\bm{1}}

\DeclareMathAlphabet{\mathsfit}{\encodingdefault}{\sfdefault}{m}{sl}
\SetMathAlphabet{\mathsfit}{bold}{\encodingdefault}{\sfdefault}{bx}{n}

\newcommand{\E}{\mathbb{E}}

\newcommand{\KL}{D_{\mathrm{KL}}}

